\documentclass[runningheads]{llncs}
\usepackage{graphicx}
\usepackage[british]{babel}
\usepackage[T1]{fontenc}
\usepackage{amssymb,amsmath,bm,mathtools}
\usepackage{nicefrac}
\usepackage{courier,mathptmx,stmaryrd}

\usepackage{hyperref,url}
  \hypersetup{  % bookmarks
    bookmarksnumbered
  }
  \hypersetup{  % links
    breaklinks=false,
    pdfborderstyle={/S/U/W 0},  % zet op .5 ofzo voor onderlijnen
    citebordercolor=.235 .702 .443,
    urlbordercolor=.255 .412 .882,
    linkbordercolor=.804 .149 .19,
  }
\usepackage{xcolor}
\usepackage{tikz}
\usepackage{enumitem}

\DeclarePairedDelimiter{\group}{(}{)}

\DeclarePairedDelimiter{\set}{\{}{\}}

\newcommand{\naturals}{\mathbb{N}}

\newcommand{\reals}{\mathbb{R}}
\newcommand{\posreals}{\reals_{>0}}

\newcommand{\states}{\mathcal{X}}

\newcommand{\gbls}{\mathcal{L}}
\newcommand{\opt}[1][]{u_{#1}}
\newcommand{\altopt}[1][]{v_{#1}}

\newcommand{\opts}{\mathcal{V}}

\newcommand{\optset}[1][]{A_{#1}}
\newcommand{\altoptset}[1][]{B_{#1}}

\newcommand{\optsets}{\mathcal{Q}}

\newcommand{\assessment}{\mathcal{A}}

\newcommand{\desirset}[1][]{D_{#1}}

\newcommand{\desirsets}{\mathbf{D}}
\newcommand{\cohdesirsets}{\overline{\desirsets}}

\newcommand{\rejectset}[1][]{K_{#1}}

\newcommand{\rejectsets}{\mathbf{K}}
\newcommand{\cohrejectsets}{\overline{\rejectsets}}

\newcommand{\choicefun}[1][]{C_{#1}}

\newcommand{\ddualopts}[1][]{\opts^\circ}

\newcommand{\ldualopts}[1][]{\underline{\opts}^\ast}

\newcommand{\nml}[1][{\opt[o]}]{\normalise_{#1}}
\newcommand{\cset}[3][]{\set[#1]{#2\colon#3}}

\newcommand{\lowprev}[1][]{\underline{P}_{#1}}
\newcommand{\uppprev}[1][]{\overline{P}_{#1}}
\newcommand{\linprev}[1][]{P_{#1}}
\newcommand{\cohlowprevs}{\underline{\mathbf{P}}}
\newcommand{\linprevs}{\mathbf{P}}
\DeclareMathOperator{\natex}{Ex}
\DeclareMathOperator{\posi}{posi}

\DeclareMathOperator{\normalise}{N}
\usepackage{thmtools,thm-restate}

\newenvironment{proof*}[1]
  {\renewcommand{\proofname}{#1}%
   \begin{proof}}
  {\end{proof}}

\usepackage{etoolbox}
\newtoggle{arxiv}
\toggletrue{arxiv}
\begin{document}
\title{Archimedean Choice Functions}%\\ an Axiomatic Foundation for Imprecise Decision Making}
\subtitle{An Axiomatic Foundation for Imprecise Decision Making}
% \author{First Author\orcidID{0000-1111-2222-3333} \and
% Second Author\orcidID{1111-2222-3333-4444}}
\author{Jasper De Bock}
\institute{
FLip, ELIS, Ghent University, Belgium\\
\email{jasper.debock@ugent.be}
}
\maketitle              % typeset the header of the contribution
\begin{abstract}
%The abstract should briefly summarize the contents of the paper in
If uncertainty is modelled by a probability measure, decisions are typically made by choosing the option with the highest expected utility. If an imprecise probability model is used instead, this decision rule can be generalised in several ways. We here focus on two such generalisations that apply to sets of probability measures: E-admissibility and maximality. Both of them can be regarded as special instances of so-called choice functions, a very general mathematical framework for decision making. For each of these two decision rules, we provide a set of necessary and sufficient conditions on choice functions that uniquely characterises this rule, thereby providing an axiomatic foundation for imprecise decision making with sets of probabilities. A representation theorem for Archimedean choice functions in terms of coherent lower previsions lies at the basis of both results.

\keywords{E-admissibility \and Maximality \and Archimedean choice functions \and Decision Making \and Imprecise probabilities.}
\end{abstract}

\section{Introduction}

Decision making under uncertainty is typically carried out by combining an uncertainty model with a decision rule. If uncertainty is modelled by a probability measure, the by far most popular such decision rule is maximising expected utility, where one chooses the option---or makes the decision---whose expected utility with respect to this probability measure is the highest.

Uncertainty can also be modelled in various other ways though. The theory of imprecise probabilities, for example, offers a wide range of extensions of probability theory that provide more flexible modelling possibilities, such as differentiating between stochastic uncertainty and model uncertainty. The most straightforward such extension is to consider a set of probability measures instead of a single one, but one can also use interval probabilities, coherent lower previsions, sets of desirable gambles, belief functions, to name only a few.

For all these different types of uncertainty models, various decision rules have been developed, making the total number of possible combinations rather daunting. Choosing which combination of uncertainty model and decision rule to use is therefore difficult and often dealt with in a pragmatic fashion, by using a combination that one is familiar with, that is convenient or that is computationaly advantageous.

To motivate the use of a specific combination in a more principled way, one can also consider its properties. That is, one can aim to select an uncertainty model and decision rule whose resulting decisions satisfy the properties that one finds desirable for the decision problem at hand. In some cases, imposing a given set of properties as axioms can even rule out all combinations but one, thereby providing an axiomatic foundation for the use of a specific type of uncertainty model and decision rule. The famous work of Savage~\cite{savage1972}, for example, provides an axiomatic foundation for maximising expected utility with respect to a probability measure. 

The main contributions of this paper are axiomatic foundations for three specific decision rules that make use of imprecise probability models~\cite{troffaes2007}. The first two decision rules, called E-admissibility~\cite{levi1974} and maximality~\cite{walley1991}, apply to a set of probability measures; they both reduce to maximising expected utility when this set contains only a single probability measure, but are otherwise fundamentally different. The third decision rule applies to sets of coherent lower previsions; it is more abstract then the other two, but includes both of them as special cases. This allows us to use our axiomatic foundation for the third rule as an intermediate step towards axiomatising E-admissibility and maximality.

To obtain our results, we make extensive use of choice functions~\cite{seidenfeld2010,2017vancamp:phdthesis,pmlr-v103-de-bock19b}: a unifying framework for studying decision making. These choice functions require no reference to an uncertainty model or a decision rule, but are simply concerned with the decisions themselves, making them an excellent tool for comparing different methods. We will be especially interested in Archimedean choice functions, because all of the decision schemes that we consider are of this particular type. 

\iftoggle{arxiv}{
  
}{
In order to adhere to the page limit constraint, all proofs 
%Proofs
are omitted; they are available in the appendix of an extended online version \cite{ipmu2020debock:arxiv}.
}

\section{Choice Functions and Uncertainty Models}

A choice function $\choicefun$, quite simply, is a function that chooses. Specifically, for every finite set $\optset$ of options, it returns a subset $\choicefun(\optset)$ of $\optset$. We here consider the special case where options are gambles: bounded real functions on some fixed state space $\states$. We let $\gbls$ be the set of all gambles on $\states$ and we use $\optsets$ to denote the set of all finite subsets of $\gbls$, including the empty set. A choice function $\choicefun$ is then a map from $\optsets$ to $\optsets$ such that, for all $\optset\in\optsets$, $\choicefun(\optset)\subseteq\optset$.

If $\choicefun(\optset)$ contains only a single option $\opt$, this means that $\opt$ is chosen from $\optset$. If $\choicefun(\optset)$ consists of multiple options, several interpretations can be adopted. On the one hand, this can be taken to mean that each of the options in $\choicefun(\optset)$ is chosen. On the other hand, $\choicefun(\optset)$ can also be regarded as a set of options among which it is not feasible to choose, in the sense that they are incomparable based on the available information; in other words: the elements of $\optset\setminus\choicefun(\optset)$ are rejected, but those in $\choicefun(\optset)$ are not necessarily `chosen'. While our mathematical results further on do not require a philosophical stance in this matter, it will become apparent from our examples and interpretation that we have the later approach in mind.

A very popular class of choice functions---while not necessarily always called as such---are those that correspond to maximising expected utility. The idea there is to consider a probability measure $\linprev$ and to let $\choicefun(\optset)$ be the element(s) of $\optset$ whose expected value---or utility---is highest. The probability measure in question is often taken to be countably additive, but we will not impose this restriction here, and impose only finite additivity. These finitely additive probability measures are uniquely characterised by their corresponding expectation operators, which are linear real functionals on $\gbls$ that dominate the infinum operator. We follow de Finetti in denoting these expectation operators by $\linprev$ as well, and in calling them linear previsions~\cite{finetti1970,walley1991}.

\begin{definition}%[Linear prevision]
\label{def:linearprevision}
A \emph{linear prevision}~$\linprev$ on~$\gbls$ is a map from~$\gbls$ to $\reals$ that satisfies
\begin{enumerate}[label=$\mathrm{P}_{\arabic*}$.,ref=$\mathrm{P}_{\arabic*}$,leftmargin=*,start=1]
\item\label{ax:linprev:inf} $\linprev(\opt\,)\geq\inf\opt$ for all $\opt\in\gbls$;\hfill \emph{boundedness}
\item\label{ax:linprev:homo} $\linprev(\lambda\opt)=\lambda\linprev(\opt)$ for all real $\lambda$ and~$\opt\in\gbls$;\hfill \emph{homogeneity}
\item\label{ax:linprev:additive} $\linprev(\opt+\altopt)=\linprev(\opt)+\linprev(\altopt)$ for all $\opt,\altopt\in\gbls$.\hfill \emph{additivity}
\end{enumerate}
We denote the set of all linear previsions on~$\gbls$ by~$\linprevs$. 
\end{definition}

For any such linear prevision---or equivalently, any finitely additive probability measure---the choice function obtained by maximising expected utility is defined by
\begin{equation}\label{eq:choicefromExpectationmax}
\choicefun[P](\optset)
\coloneqq
\big\{\opt\in\optset\colon(\forall\altopt\in\optset\setminus\{\opt\})~\linprev(\opt)\geq\linprev(\altopt)\big\}
\text{~~for all }\optset\in\optsets.
\end{equation}
It returns the options $\opt$ in $\optset$ that have the highest prevision---or expectation---$\linprev(\opt)$.

However, there are also many situations in which it is not feasible to represent uncertainty by a single prevision or probability measure~\cite[Section~1.4.4]{walley1991}. In those cases, imprecise probability models can be used instead. The most straightforward such imprecise approach is to consider a non-empty set $\mathcal{P}\subseteq\linprevs$ of linear previsions---or probability measures---as uncertainty model, the elements of which can be regarded as candidates for some `true' but unknown precise model.

In that context, maximising expected utility can be generalised in several ways~\cite{troffaes2007}, of which we here consider two. The first is called E-admissibility; it chooses those options that maximise expected utility with respect to at least one precise model in $\mathcal{P}$:
\begin{align}
\choicefun[\mathcal{P}]^{\mathrm{\,E}}(\optset)
\coloneqq&
%\bigcup\big\{\choicefun[\linprev](\optset)\colon\linprev\in\mathcal{P}\big\}\notag\\
%=&
\big\{\opt\in\optset\colon(\exists\linprev\in\mathcal{P})\,(\forall\altopt\in\optset\setminus\{\opt\})~\linprev(\opt)\geq\linprev(\altopt)\big\}
\text{~~for all }\optset\in\optsets.
\label{eq:choicefromEadmissibility}
\end{align}
The second generalisation is called maximality and starts from a partial order on the elements of $\optset$. In particular, for any two options $\opt,\altopt\in\optset$, the option $\opt$ is deemed better than $\altopt$ if its expectation is higher for every $\linprev\in\mathcal{P}$. Decision making with maximality then consists in choosing the options $\opt$ in $\optset$ that are undominated in this order, in the sense that no other option $\altopt\in\optset$ is better than $\opt$:
\begin{equation}\label{eq:choicefromMaximality}
\choicefun[\mathcal{P}]^{\mathrm{\,M}}(\optset)
\coloneqq
\big\{\opt\in\optset\colon(\forall\altopt\in\optset\setminus\{\opt\})\,(\exists\linprev\in\mathcal{P})~\linprev(\opt)\geq\linprev(\altopt)\big\}
\text{~~for all }\optset\in\optsets.
\end{equation}
One can easily verify that $\choicefun[\mathcal{P}]^{\mathrm{\,E}}(\optset)\subseteq\choicefun[\mathcal{P}]^{\mathrm{\,M}}(\optset)$, making maximality the most conservative decisions rule of the two.
Furthermore, in the particular case where $\mathcal{P}$ contains only a single linear prevision, they clearly coincide and both reduce to maximising expected utility. In all other cases, however, maximality and E-admissibility are different; see for example Proposition~\ref{prop:EisMiffsingle} in Section~\ref{sec:axiomatisingEM} for a formal statement.

One of the main aims of this paper is to characterise each of these two types of choice functions in terms of their properties. That is, we are looking for necessary and sufficient conditions under which a general choice function $\choicefun$ is of the form $\choicefun[\mathcal{P}]^{\mathrm{\,E}}$ or $\choicefun[\mathcal{P}]^{\mathrm{\,M}}$, without assuming a priori the existence of a set of linear previsions $\mathcal{P}$. Such conditions will be presented in Section~\ref{sec:Eadmissibility} and~\ref{sec:Maximality}, respectively.

A crucial intermediate step in obtaining these two results will consist in finding a similar characterisation for choice functions that correspond to (sets of) coherent lower previsions~\cite{walley1991}, a generalisation of linear previsions that replaces additivity by the weaker property of superadditivity.
\begin{definition}%[Coherent lower prevision]
\label{def:lowerprevision}
A \emph{coherent lower prevision}~$\lowprev$ on~$\gbls$ is a map from~$\gbls$ to $\reals$ that satisfies
\begin{enumerate}[label=$\mathrm{LP}_{\arabic*}$.,ref=$\mathrm{LP}_{\arabic*}$,leftmargin=*,start=1]
\item\label{ax:lowprev:inf} $\lowprev(\opt)\geq\inf\opt$ for all $\opt\in\gbls$;\hfill \emph{boundedness}
\item\label{ax:lowprev:homo} $\lowprev(\lambda\opt)=\lambda\lowprev(\opt)$ for all real $\lambda>0$ and $\opt\in\gbls$;\hfill \emph{positive homogeneity}
\item\label{ax:lowprev:superadditive} $\lowprev(\opt+\altopt)\geq\lowprev(\opt)+\lowprev(\altopt)$ for all $\opt,\altopt\in\gbls$.\hfill \emph{superadditivity}
\end{enumerate}
We denote the set of all coherent lower previsions on~$\gbls$ by~$\cohlowprevs$.
\end{definition}

That linear previsions are a special case of coherent lower previsions follows trivially from their definitions. There is however also a more profound connection between both concepts: coherent lower previsions are minima of linear ones.

\begin{theorem}\label{theo:lowerenvelop}
\cite[Section 3.3.3.]{walley1991}
A real-valued map $\lowprev$ on $\gbls$ is a coherent lower prevision if and only if there is a non-empty set $\mathcal{P}\subseteq\linprevs$ of linear previsions such that
\begin{equation*}
\lowprev(\opt)=\min\{\linprev(\opt)\colon\linprev\in\mathcal{P}\}\text{ for all $\opt\in\gbls$.}
\end{equation*}
\end{theorem}

\noindent
Alternatively, coherent lower previsions can also be given a direct behavioural interpretation in terms of gambling, without any reference to probability measures or linear previsions~\cite{walley1991,williams2007}.

Regardless of their interpretation, with any given non-empty set $\mathcal{P}\subseteq\cohlowprevs$ of these coherent lower previsions, we can associate a choice function $\choicefun[\mathcal{P}]$ in the following way:
\begin{equation}\label{eq:archimedeanchoicefunction}
\choicefun[\mathcal{P}](\optset)
\coloneqq
\big\{\opt\in\optset\colon(\exists\lowprev\in\mathcal{P})\,(\forall\altopt\in\optset\setminus\{\opt\})~\lowprev(\altopt-\opt)\leq0\big\}
\text{ for all $\optset\in\optsets$.}
\end{equation}
If the lower previsions in $\mathcal{P}$ are all linear, this definition reduces to E-admissibility, as can be seen by comparing Equations~\eqref{eq:choicefromEadmissibility} and~\eqref{eq:archimedeanchoicefunction}. What is far less obvious though, is that maximality is also a special case of Equation~\eqref{eq:archimedeanchoicefunction}; see Theorem~\ref{theo:axiomatisationforMaximality} in Section~\ref{sec:Maximality} for a formal statement. In that case, the sets $\mathcal{P}$ in Equations~\eqref{eq:archimedeanchoicefunction} and~\eqref{eq:choicefromMaximality} may of course---and typically will---be different.

Because the choice functions that correspond to E-admissibility and maximality are both of the form $\choicefun[\mathcal{P}]$, with $\mathcal{P}$ a set of coherent lower previsions, any attempt at characterising the former will of course benefit from characterising the latter. A large part of this paper will therefore be devoted to the development of necessary and sufficient conditions for a general choice function $\choicefun$ to be of the form $\choicefun[\mathcal{P}]$. In order to obtain such conditions, we will interpret choice functions in terms of (strict) desirability and establish a connection with so-called sets of desirable option sets. This interpretation will lead to a natural set of conditions that, as we will eventually see in Section~\ref{sec:Archimedeanchoicefunctions}, uniquely characterises choice functions of the form $\choicefun[\mathcal{P}]$. We start with a brief introduction to desirability and sets of desirable option sets.

\section{Coherent Sets of Desirable Option Sets}\label{sec:rejects}

The basic notion on which our interpretation for choice functions will be based, and from which our axiomatisation will eventually be derived, is that of a desirable option: an option that is strictly preferred over the status quo~\cite{couso2011,quaeghebeur2015:statement,walley1991}. In our case, where options are gambles $\opt\in\gbls$ on $\states$ and the status quo is the zero gamble, this means that the uncertain---and possibly negative---reward $\opt(x)$, whose actual value depends on the uncertain state $x\in\states$, is strictly preferred over the zero reward. In other words: gambling according to $\opt$ is strictly preferred over not gambling at all.

We will impose the following three principles on desirable options, where we use `$(\lambda,\mu)>0$' as a shorthand notation for `$\lambda\geq0$, $\mu\geq0$ and $\lambda+\mu>0$'. The first two principles follow readily from the meaning of desirability. The third one follows from an assumption that rewards are expressed in a linear utility scale.

\begin{enumerate}[label=$\mathrm{d}_{\arabic*}$.,ref=$\mathrm{d}_{\arabic*}$,start=1,leftmargin=*,topsep=5pt,%itemsep=2pt
]
\item\label{ax:desirability:nozero} 
$0$ is not desirable;
\item\label{ax:desirability:pos} if $\inf\opt>0$, then $\opt$ is desirable;\footnote{There is no consensus on which properties to impose on desirability; the main ideas and results are always very similar though~\cite{couso2011,quaeghebeur2015:statement,walley1991,williams2007}. In particular, \ref{ax:desirability:pos} is often strengthened by requiring that $\opt$ is desirable as soon as $\inf\opt\geq0$ and $\opt\neq0$; we here prefer \ref{ax:desirability:pos} because it is less stringent and because it combines more easily with the notion of strict desirability that we will consider in Section~\ref{sec:archimedeanrejects}. Annalogous comments apply to~\ref{ax:desirs:pos} and~\ref{ax:rejects:pos} further on.}
\item\label{ax:desirability:cone} if $\opt,\altopt$ are desirable and $(\lambda,\mu)>0$, then $\lambda\opt+\mu\altopt$ is desirable.
\end{enumerate}

The notion of a desirable option gives rise to two different frameworks for modelling a subject's uncertainty about the value $x\in\states$. The first, which is well established, is that of sets of desirable options---or sets of desirable gambles. The idea there is to consider a set $\desirset$ that consists of options that are deemed desirable by a subject. If such a set is compatible with the principles \ref{ax:desirability:nozero}--\ref{ax:desirability:cone}, it is called coherent.

\begin{definition}%[Coherence for sets of desirable options]
\label{def:cohdesir}
A set of desirable options $\desirset\in\desirsets$ is \emph{coherent} if it satisfies: 
\begin{enumerate}[label=$\mathrm{D}_{\arabic*}$.,ref=$\mathrm{D}_{\arabic*}$,leftmargin=*]
\item\label{ax:desirs:nozero} $0\notin\desirset$;
\item\label{ax:desirs:pos} if\, $\inf\opt>0$, then $\opt\in\desirset$;
\item\label{ax:desirs:cone} if $\opt,\altopt\in\desirset$ and $(\lambda,\mu)>0$, then $\lambda\opt+\mu\altopt\in\desirset$.
\end{enumerate}
We denote the set of all coherent sets of desirable options by $\cohdesirsets$.
\end{definition}

A more general framework, which will serve as our main workhorse in this paper, is that of sets of desirable option sets~\cite{debock2018}. The idea here is to consider a set $\rejectset$ of so-called desirable option sets $\optset$, which are finite sets of options that, according to our subject, are deemed to contain at least one desirable option. To say that $\optset=\{\opt,\altopt\}$ is a desirable option set, for example, means that $\opt$ or $\altopt$ is desirable. Crucially, the framework of sets of desirable option sets allows a subject to make this statement without having to specify---or know---which of the two options $\opt$ or $\altopt$ is desirable. As explained in earlier work \cite[Section~3]{pmlr-v103-de-bock19b},\footnote{Reference~\cite{pmlr-v103-de-bock19b} deals with the more general case where options take values in an abstract vector space $\mathcal{V}$, and where \ref{ax:desirability:pos} imposes that $u$ should be desirable if $u\succ0$, with $\succ$ an arbitrary but fixed strict vector ordering. Whenever we invoke results from~\cite{pmlr-v103-de-bock19b}, we are applying them for the special case where $\mathcal{V}=\gbls$ and $\opt\succ\altopt\Leftrightarrow\inf(\opt-\altopt)>0$.}
 it follows from \ref{ax:desirability:nozero}--\ref{ax:desirability:cone} that any set of desirable option sets $\rejectset$ should satisfy the following axioms. If it does, we call $\rejectset$ coherent.

\begin{definition}%[Coherence for sets of desirable option sets]
\label{def:coherence:rejectset}
A set of desirable option sets $\rejectset\subseteq\optsets$ is \emph{coherent} if it satisfies:
\begin{enumerate}[label=$\mathrm{K}_{\arabic*}$.,ref=$\mathrm{K}_{\arabic*}$,leftmargin=*,start=0]
\item\label{ax:rejects:removezero} if $\optset\in\rejectset$ then also $\optset\setminus\set{0}\in\rejectset$, for all $\optset\in\optsets$;
\item\label{ax:rejects:nozero} $\set{0}\notin\rejectset$;
\item\label{ax:rejects:pos} $\set{\opt}\in\rejectset$ for all $\opt\in\gbls$ with $\inf\opt>0$;
\item\label{ax:rejects:cone} if $\optset[1],\optset[2]\in\rejectset$ and if, for all $\opt\in\optset[1]$ and $\altopt\in\optset[2]$, $(\lambda_{\opt,\altopt},\mu_{\opt,\altopt})>0$, then also
\begin{equation*}
\cset{\lambda_{\opt,\altopt}\opt+\mu_{\opt,\altopt}\altopt}{\opt\in\optset[1],\altopt\in\optset[2]}
\in\rejectset;
\end{equation*}
\item\label{ax:rejects:mono} if $\optset[1]\in\rejectset$ and $\optset[1]\subseteq\optset[2]$, then also $\optset[2]\in\rejectset$, for all $\optset[1],\optset[2]\in\optsets$.
\end{enumerate}
We denote the set of all coherent sets of desirable option sets by $\cohrejectsets$.
\end{definition}

One particular way of obtaining a set of desirable option sets, is to derive it from a set of desirable options $\desirset$, as follows:
\begin{equation}\label{eq:desirset:to:rejectset}
\rejectset[\desirset]
\coloneqq\cset{\optset\in\optsets}{\optset\cap\desirset\neq\emptyset}.
\end{equation}
One can easily verify that if $\desirset$ is coherent, then $\rejectset[\desirset]$ will be as well~\cite[Proposition~8]{pmlr-v103-de-bock19b}. In general, however, sets of desirable option sets are more expressive than sets of desirable options. The link between both is provided by Theorem~\ref{theo:coherentrejectsets:representation}, which shows that a set of desirable option sets can be equivalently represented by a set of sets of desirable options.

\begin{theorem}%[Representation for coherent sets of desirable option sets]
\label{theo:coherentrejectsets:representation}
\cite[Theorem~9]{pmlr-v103-de-bock19b} 
A set of desirable option sets $\rejectset$ is coherent if and only if there is some non-empty set $\mathcal{D}\subseteq\cohdesirsets$ %of coherent sets of desirable options 
such that $\rejectset=\bigcap\{\rejectset[\desirset]\colon\desirset\in\mathcal{D}\}$. 
\end{theorem}

In practice, modelling a subject's uncertainty does not require her to specify a full coherent set of desirable option sets though. Instead, it suffices for her to provide an assessment $\assessment\subseteq\optsets$, consisting of option sets $\optset$ that she considers desirable. If such an assessment is consistent with coherence, meaning that there is at least one coherent set of desirable option sets $\rejectset$ that includes $\assessment$, then this assessment can always be extended to a unique smallest---most conservative---coherent set of desirable option sets, called the natural extension of $\assessment$. This natural extension is given by
\begin{equation}\label{eq:natex}
\natex(\assessment)\coloneqq
\bigcap\big\{\rejectset\in\cohrejectsets\colon\assessment\subseteq\rejectset\big\}
=
\bigcap\big\{\rejectset[\desirset]\colon\desirset\in\cohdesirsets, \assessment\subseteq\rejectset[\desirset]\big\},
\end{equation}
as follows readily from Theorem~\ref{theo:coherentrejectsets:representation}.
If $\assessment$ is not consistent with coherence, $\natex(\assessment)$ is an empty intersection, which, by convention, we set equal to $\optsets$.

\section{Strongly Archimedean Sets of Desirable Option Sets}

That sets of desirable option sets can be used to axiomatise choice functions of the type $\choicefun[\mathcal{P}]$, with $\mathcal{P}$ a set of coherent lower previsions, was already demonstrated in earlier work~\cite{pmlr-v103-de-bock19b} for the specific case where $\mathcal{P}$ is closed with respect to pointwise convergence. A key step in that result consisted in strengthening the interpretation of $\rejectset$, replacing desirability with the stronger notion of strict desirability~\cite[Section 3.7.7]{walley1991}. We here repeat the reasoning that led to this result, before adapting it in Section~\ref{sec:archimedeanrejects} to get rid of the closure condition.

We call a desirable option $\opt$ strictly desirable if there is some real $\epsilon>0$ such that $\opt-\epsilon$ is desirable.\footnote{Walley's original notion of strict desirability~\cite[Section 3.7.7]{walley1991} is slightly different. In his version, if $\inf\opt=0$ (but $\opt\neq0$) then $\opt$ should also be strictly desirable (but need not satisfy~\ref{ax:strictdesirability:epsilon}).} As a simple consequence of this definition and \ref{ax:desirability:nozero}--\ref{ax:desirability:cone}, we find that
\begin{enumerate}[label=$\mathrm{sd}_{\arabic*}$.,ref=$\mathrm{sd}_{\arabic*}$,start=1,leftmargin=*,topsep=5pt,itemsep=2pt]
\item\label{ax:strictdesirability:nozero} 
$0$ is not strictly desirable;
\item\label{ax:strictdesirability:pos} if $\inf\opt>0$, then $\opt$ is strictly desirable;
\item\label{ax:strictdesirability:cone} if $\opt,\altopt$ are strictly desirable and $(\lambda,\mu)>0$, then $\lambda\opt+\mu\altopt$ is strictly desirable;
\item\label{ax:strictdesirability:epsilon}
if $\opt$ is strictly desirable, then $\opt-\epsilon$ is strictly desirable for some real $\epsilon>0$.
\end{enumerate}

\noindent
By applying these principles to sets of desirable options, we arrive at the concept of a coherent set of strictly desirable options: a coherent set of desirable options $\desirset$ that is compatible with~\ref{ax:strictdesirability:epsilon}. What is particularly interesting about such sets is that they are in one-to-one correspondence with coherent lower previsions \cite{pmlr-v103-de-bock19b,walley1991}, thereby allowing us to move from desirability to lower previsions as a first step towards choice functions of the form $\choicefun[\mathcal{P}]$. The problem with coherent sets of strictly desirable options, however, is that they correspond to a single lower prevision $\lowprev$, while we which to consider a set $\mathcal{P}$ of them. To achieve this, we again consider sets of desirable option sets, but now suitably adapted to strict desirability.

So consider any set of desirable option sets $\rejectset$ and let us interpret it in terms of strict desirability. That $\optset$ belongs to $\rejectset$ then means that $\optset$ contains at least one strictly desirable option. Given this interpretation, what properties should $\rejectset$ satisfy? Since the principles \ref{ax:strictdesirability:nozero}--\ref{ax:strictdesirability:cone} are identical to \ref{ax:desirability:nozero}--\ref{ax:desirability:cone}, $\rejectset$ should clearly be coherent, meaning that it should satisfy \ref{ax:rejects:removezero}--\ref{ax:rejects:mono}. Formalising the implications of \ref{ax:strictdesirability:epsilon} is more tricky though, as it can be done in several ways.

The first and most straighforward approach is to impose the following immediate translation of \ref{ax:strictdesirability:epsilon} to desirable option sets, where for all $\optset\in\optsets$ and real $\epsilon$:
\begin{equation*}
\optset-\epsilon\coloneqq\{\opt-\epsilon\colon\opt\in\optset\}
\end{equation*}

\begin{definition}%[Strong Archimedeanity for sets of desirable option sets]
\label{def:infinite:strictcoherence}
A set of desirable option sets $\rejectset$ is \emph{strongly Archimedean}\footnote{In earlier work~\cite{pmlr-v103-de-bock19b}, we have referred to this property as Archimedeanity. With hindsight, however, we now prefer to reserve this terminology for the property in Definition~\ref{def:infinite:strictcoherence}.} if it is coherent and satisfies the following property:
\begin{enumerate}[label=$\mathrm{K}_{\mathrm{SA}}$.,ref=$\mathrm{K}_{\mathrm{SA}}$,leftmargin=*]
\item\label{ax:rejectsets:sa}
if $\optset\in\rejectset$, then also $\optset-\epsilon\in\rejectset$ for some real $\epsilon>0$.
\end{enumerate}
\end{definition}

The reasoning behind this axiom goes as follows. Since $\optset\in\rejectset$ is taken to mean that there is at least one $\opt\in\optset$ that is strictly desirable, it follows from \ref{ax:strictdesirability:epsilon} that there is some real $\epsilon>0$ such that $\opt-\epsilon$ is strictly desirable. This implies that $\optset-\epsilon$ contains at least one strictly desirable option. It therefore seems sensible to impose that $\optset-\epsilon\in\rejectset$. 

To explain the implications of this axiom, and how it is related to lower previsions, we need a way to link the latter to sets of desirable option sets. The first step is to associate, with any coherent lower prevision $\lowprev\in\cohlowprevs$, a set of desirable option sets
\begin{equation}\label{eq:Kfromlowprev}
\rejectset[\lowprev]\coloneqq\cset{\optset\in\optsets}{(\exists\opt\in\optset)\lowprev(\opt\,)>0}.
\end{equation}
The coherence of this set can be easily verified~\cite[Propositions 8 and 24]{pmlr-v103-de-bock19b}. More generally, with any non-empty set $\mathcal{P}$ of coherent lower previsions, we associate a set of desirable option sets\vspace{-3pt}
\begin{equation}\label{eq:fromsetofLPstoK}
\rejectset[\mathcal{P}]\coloneqq\bigcap\cset{\rejectset[\lowprev]}{\lowprev\in\mathcal{P}}.\\[5pt]
\end{equation}
Coherence is again easily verified; it follows directly from the coherence of $\rejectset[\lowprev]$ and the fact that coherence is preserved under taking intersections. The final tool that we need to explain the implications of strong Archimedeanity, does the opposite; it starts with a coherent set of desirable option sets $\rejectset$, and associates a set of coherent lower previsions with it, defined by
\begin{equation*}
\cohlowprevs(\rejectset)\coloneqq\cset{\lowprev\in\cohlowprevs}{\rejectset\subseteq\rejectset[\lowprev]}.\\[2pt]
\end{equation*}
If $\rejectset$ is strongly Archimedean, then as the following result shows, $\cohlowprevs(\rejectset)$ serves as a mathematically equivalent representation for $\rejectset$, from wich $\rejectset$ can be recovered through Equation~\eqref{eq:fromsetofLPstoK}. The representing set $\cohlowprevs(\rejectset)$ will then furthermore be closed with respect to the topology induced by pointwise convergence.

\begin{theorem}\label{theo:rejectsets:representation:lowerprev:twosidedsa}\cite[Theorem 28 and Proposition 24]{pmlr-v103-de-bock19b}
A set of desirable option sets $\rejectset$ is strongly Archimedean if and only if there is some non-empty closed set $\mathcal{P}\subseteq\cohlowprevs$ of coherent lower previsions such that $\rejectset=\rejectset[\mathcal{P}]$. Closure is with respect to pointwise convergence, and the largest such set~$\mathcal{P}$ is then\/ $\cohlowprevs(\rejectset)$.
\end{theorem}

If the representing coherent lower previsions in $\mathcal{P}$ or $\cohlowprevs(\rejectset)$ were linear, this result would already brings us very close to decision rules based on sets of linear previsions---or sets of probability measures. As we will see further on in Section~\ref{sec:Eadmissibility}, this can be achieved by imposing an additional axiom called mixingness. Before we do so, however, we will do away with the closure condition in Theorem~\ref{theo:rejectsets:representation:lowerprev:twosidedsa}, as it is overly restrictive. Imagine for example that we are modelling a subject's uncertainty about the outcome of a coin toss, and that she beliefs the coin to be unfair. 
In terms of probabilities, this would mean that her probability for heads is different from one half. Strong Archimeanity is  not compatible with such an assessment, as the set of probability measures that satisfy this (strict) probability constraint is not closed. Our first main contribution will consist in resolving this issue, by suitably modifying the notion of strong Archimedeanity.

\section{Archimedean Sets of Desirable Option Sets}\label{sec:archimedeanrejects}

At first sight, it may seem as if \ref{ax:rejectsets:sa} is the only way in which \ref{ax:strictdesirability:epsilon} can be translated to option sets. There is however also a second, far more subtle approach.

The crucial insight on which this second approach is based is that our interpretation in terms of strict desirability does not require $\epsilon$ to be known; it only imposes the existence of such an $\epsilon$. 
Consider a subject whose uncertainty is represented by a set of desirable option sets $\rejectset$ and let us adopt an interpretation in terms of strict desirability. This implies that the option sets $\optset\in\rejectset$ are option sets that, according to her beliefs, contain at least one strictly desirable option $\opt\in\optset$. As a consequence of \ref{ax:strictdesirability:epsilon}, this implies that she beliefs that there is some real $\epsilon>0$ such that $\opt-\epsilon$ is strictly desirable. Hence, she believes that there is some real $\epsilon>0$ such that $\optset-\epsilon$ contains at least one strictly desirable option. Strong Archimedeanity, at that point, concludes that $\optset-\epsilon\in\rejectset$. However, this is only justified if our subject knows $\epsilon$. If she doesn't know $\epsilon$, but only believes that there is such an $\epsilon$, then there is no single $\epsilon>0$ for which she believes that $\optset-\epsilon$ contains at least one strictly desirable option. Since the option sets in $\rejectset$ are options sets for which our subject believes that they contain at least one strictly desirable option, it follows that $\optset\in\rejectset$ need not necessarily imply that $\optset-\epsilon\in\rejectset$ for some $\epsilon>0$. Strong Archimedeanity is therefore indeed, as its name suggests, a bit too strong for our purposes.

So if we can't infer that $\optset-\epsilon\in\rejectset$, what is it then that we can infer from $\optset\in\rejectset$ and \ref{ax:strictdesirability:epsilon}?
As explained above, the only thing that can be inferred is that for any $\optset\in\rejectset$, there is some $\epsilon>0$ such that $\optset-\epsilon$ contains at least one strictly desirable option. Let us denote this specific epsilon by $\epsilon(\optset)$. Crucially, we may not know---or rather, our subject may not know---the specific value of $\epsilon(\optset)$. Nevertheless, any inferences we can make without knowing the specific value of $\epsilon(\optset)$, can and should still be made. Our approach will therefore consist in finding out what inferences can be made for a specific choice of the $\epsilon(\optset)$, to do this for every such choice, and to then only consider those inferences that can be made regardless of the specific choice of $\epsilon(\optset)$.

To formalize this, we consider the set $\posreals^{\rejectset}$ of all functions $\epsilon$ that associate a strictly positive real $\epsilon(\optset)>0$ with every option set $\optset$ in $\rejectset$. As a consequence of our interpretation, we know that there is at least one $\epsilon\in\posreals^{\rejectset}$ such that, for every $\optset\in\rejectset$, $\optset-\epsilon(\optset)$ contains a strictly desirable option.

Let us now assume for a minute that our subject does know for which specific $\epsilon$ in $\posreals^{\rejectset}$ this is the case. 
In order to be compatible with \ref{ax:strictdesirability:nozero}--\ref{ax:strictdesirability:cone}, the resulting assessment
\begin{equation}\label{eq:Kepsilon}
\rejectset[\epsilon]\coloneqq\{\optset-\epsilon(\optset)\colon\optset\in\rejectset\}
\end{equation}
should then be consistent with coherence, meaning that there is at least one coherent set of desirable option sets that includes $\rejectset[\epsilon]$. Whenever this is the case, then as explained in Section~\ref{sec:rejects}, we can use coherence to extend the assessment $\rejectset[\epsilon]$ to the unique smallest coherent set of desirable option sets that incudes it: the natural extension $\natex(\rejectset[\epsilon])$ of $\rejectset[\epsilon]$. Based on the assessment $\rejectset[\epsilon]$ and coherence, each of the option sets in $\natex(\rejectset[\epsilon])$ must necessarily contain a strictly desirable option. Hence, still assuming for the moment that our subject knows $\epsilon$, it follows that every option set in $\natex(\rejectset[\epsilon])$ should belong to $\rejectset$.

Our subject may not know $\epsilon$ though; all we can infer from \ref{ax:strictdesirability:epsilon} is that there must be at least one $\epsilon$ for which the above is true. Let us denote this specific---but possibly unkown---$\epsilon$ by $\epsilon^*$. Then as argued above, for every option set $\optset$ in $\natex(\rejectset[\epsilon^*])$, it follows from our interpretation that $\optset$ should also belong to $\rejectset$. Since we don't know $\epsilon^*$, however, we don't know for which option sets this is the case. What we can do though, is to consider those option sets $\optset\in\optsets$ that belong to $\natex(\rejectset[\epsilon])$ for \emph{every} possible $\epsilon\in\posreals^{\rejectset}$. For those option sets, regardless of whether we know $\epsilon^*$ or not, it trivially follows that $\optset\in\natex(\rejectset[\epsilon^*])$, and therefore, that $\optset$ should belong to $\rejectset$. Any coherent set of desirable option sets $\rejectset$ that satisfies this property, we will call Archimedean.

\begin{definition}%[Archimedeanity for sets of desirable option sets]
\label{def:infinite:strictcoherence}
A set of desirable option sets $\rejectset$ is \emph{Archimedean} if it is coherent and satisfies the following property:
\begin{enumerate}[label=$\mathrm{K}_{\mathrm{A}}$.,ref=$\mathrm{K}_{\mathrm{A}}$,leftmargin=*]
\item\label{ax:infiniteoptionsets:evencontinuity}
for any $\optset\in\optsets$, if $\optset\in\natex(\rejectset[\epsilon])$ for all $\epsilon\in\posreals^{\rejectset}$, then also $\optset\in\rejectset$.
\end{enumerate} 
% The set of all Archimedean sets of desirable option sets is denoted by $\strictcohrejectsets$.
\end{definition}

\noindent
Note that Archimedeanity also rules out the possibility that $\rejectset[\epsilon]$ is inconsistent for all $\epsilon\in\posreals^{\rejectset}$, for this would imply that
$\rejectset=\optsets$, hence contradicting~\ref{ax:rejects:nozero}.

By replacing strong Archimedeanity with Archimedeanity, the condition that the representing set $\mathcal{P}$ must be closed can be removed from Theorem~\ref{theo:rejectsets:representation:lowerprev:twosidedsa}, and we obtain a representation in terms of general sets of lower previsions.

\begin{theorem}%[Representation for Archimedean choice functions]
\label{theo:rejectsets:representation:lowerprev:twosidedstrict}
A set of desirable option sets $\rejectset$ is Archimedean if and only if there is some non-empty set $\mathcal{P}\subseteq\cohlowprevs$ of coherent lower previsions such that $\rejectset=\rejectset[\mathcal{P}]$. 
The largest such set~$\mathcal{P}$ is then\/ $\cohlowprevs(\rejectset)$.
\end{theorem}

\noindent
The significance of this result is that it relates two very different things: sets of desirable option sets and sets of coherent lower previsions. While this may not yet be obvious, this is a major step in characterising choice functions of the form $\choicefun[\mathcal{P}]$. In fact, we are nearly there. The only thing left to do is to connect choice functions with sets of desirable option sets. As we will explain in the next section, this connection comes quite naturally once we interpret choice functions in terms of (strict) desirability.

\section{Archimedean Choice Functions}\label{sec:Archimedeanchoicefunctions}

In order to provide choice functions with an interpretation, we need to explain what it means for an option $\opt$ to be chosen from $\optset$, or alternatively, what it means for $\opt$ to be rejected from $\optset$, in the sense that $\opt\notin\choicefun(\optset)$. We here adopt the latter approach. In particular, if our subject states that $\opt\notin\choicefun(\optset)$, we take this to mean that she is convinced that there is at least one other option $\altopt$ in $\optset\setminus\{\opt\}$ that is better that $\opt$, where `$\altopt$ is better than $\opt$' is taken to mean that $\altopt-\opt$ is strictly desirable, or equivalently, that there is a positive price $\epsilon>0$ for which paying $\epsilon$ to exchange the uncertain reward $\opt$ for $\altopt$ is preferrable to the status quo. Note however that this interpretation does not not assume that our subject knows the specific $\epsilon$ and $\altopt$ for which this is the case.

Our interpretation has two implications for $\choicefun$. First, since $\altopt-\opt=(\altopt-\opt)-0$, it immediately implies that $\choicefun$ should be \emph{translation invariant}, in the sense that
\begin{equation}\label{eq:translationinvariance}
\opt\in\choicefun(\optset)
\Leftrightarrow
0\in\choicefun(\optset-\opt)
\text{ for all $\optset\in\optsets$ and $\opt\in\optset$,}
\end{equation}
with $\optset-\opt\coloneqq\{\altopt-\opt\colon\altopt\in\optset\}$.
Second, for all $\optset\in\optsets$ such that $0\notin\choicefun(\optset\cup\{0\})$, it implies that $\optset$ should contain at least one strictly desirable gamble. Indeed, if $0\notin\choicefun(\optset\cup\{0\})$, then according to our interpretation, there is some $\altopt\in(\optset\cup\{0\})\setminus\{0\}\subseteq\optset$ such that $\altopt-0=\altopt$ is strictly desirable. Hence, $\optset$ indeed contains a strictly desirable option. For any choice function $\choicefun$, this leads us to consider the set of desirable option sets
\begin{equation}\label{eq:fromCtoK}
\rejectset[\choicefun]
\coloneqq
\big\{
  \optset\in\optsets\colon0\notin\choicefun(\optset\cup\{0\})
\big\}.
\end{equation}
According to our interpretation, each of the option sets in $\rejectset[\choicefun]$ contains at least one strictly desirable option.
Following the discussion in Section~\ref{sec:Archimedeanchoicefunctions}, we will therefore require $\rejectset[\choicefun]$ to be Archimedean. When a choice function $\choicefun$ satisfies both of the conditions that are implied by our interpretation, we call it Archimedean.

\begin{definition}
A choice function $\choicefun$ is Archimedean if $\rejectset[\choicefun]$ is Archimedean and $\choicefun$ is translation invariant.
\end{definition}

Instead of deriving a set of desirable options sets $\rejectset[\choicefun]$ from a choice function $\choicefun$, we can also do the converse. That is, with any set of desirable option sets $\rejectset$, we can associate a choice function $\choicefun[\rejectset]$, defined by
\begin{equation}\label{eq:fromKtoC}
\choicefun[\rejectset](\optset)
\coloneqq
\big\{
  \opt\in\optset\colon\optset\ominus\opt\notin\rejectset
\big\}
\text{ for all $\optset\in\optsets$,}
\end{equation}
where $\optset\ominus\opt\coloneqq\{\altopt-\opt\colon\altopt\in\optset\setminus\{\opt\}\}$. Similarly to $\rejectset[\choicefun]$, the expression for $\choicefun[\rejectset]$ is motivated by our interpretation. Indeed, for any option $\opt\in\optset$, the statement that $\optset\ominus\opt\in\rejectset$ means that $\optset\ominus\opt$ contains a strictly desirable option, so there is some $\altopt\in\optset\setminus\{\opt\}$ such that $\altopt-\opt$ is strictly desirable. This is exactly our interpretation for $\opt\notin\choicefun(\optset)$.

If a set of desirable option sets $\rejectset$ is Archimedean, then $\choicefun[\rejectset]$ will be Archimedean as well. In fact, as our next result shows, every Archimedean choice function is of the form $\choicefun[\rejectset]$, with $\rejectset$ an Archimedean set of desirable option sets.

\begin{proposition}\label{prop:strictlycoherentCiffK}
Let $\choicefun$ be a choice function. Then $\choicefun$ is Archimedean if and only if there is an Archimedean set of desirable option sets $\rejectset$ such that $\choicefun=\choicefun[\rejectset]$. This set $\rejectset$ is then necessarily unique and furthermore equal to $\rejectset[\choicefun]$.
\end{proposition}

At this point, the hard work in characterising choice functions of the form $\choicefun[\mathcal{P}]$ is done. Proposition \ref{prop:strictlycoherentCiffK} relates Archimedean choice functions to Archimedean sets of desirable option sets, while Theorem~\ref{theo:rejectsets:representation:lowerprev:twosidedstrict} relates Archimedean sets of desirable option sets to sets of coherent lower previsions $\mathcal{P}$. Combining both results, we find that a choice function is Archimedean if and only if it is of the form $\choicefun[\mathcal{P}]$.

\begin{theorem}\label{theo:axiomatisationofArchimedeanity}
A choice function $\choicefun$ is Archimedean if and only if there is a non-empty set $\mathcal{P}\subseteq\cohlowprevs$ of coherent lower previsions such that $\choicefun=\choicefun[\mathcal{P}]$. 
Whenever this is the case, the largest such set $\mathcal{P}$ is $\cohlowprevs(\rejectset[\choicefun])$.
\end{theorem}

Starting from this result, we will now proceed to axiomatise maximality and E-admissibility, by combining Archimedeanity with additional axioms.

\section{Axiomatising E-admissibility}\label{sec:Eadmissibility}

Archimedeanity implies that a choice function is representable by a set of coherent lower previsions $\mathcal{P}$, in the sense that is of the form $\choicefun[\mathcal{P}]$. As can be seen by comparing Equations~\eqref{eq:archimedeanchoicefunction} and~\eqref{eq:choicefromEadmissibility}, this already brings us very close E-admissibility. Indeed, all that we need in order to obtain E-admissibility is for the lower previsions in $\mathcal{P}$ to be linear. That is, we would like the role of $\cohlowprevs(\rejectset)$ to be taken up by
\begin{equation*}
\linprevs(\rejectset)\coloneqq\cset{\linprev\in\linprevs}{\rejectset\subseteq\rejectset[\linprev]}
\end{equation*}
instead.
To achieve this, we impose a property called mixingness~\cite{pmlr-v103-de-bock19b} on the Archimedean set of desirable option sets $\rejectset[\choicefun]$ that corresponds to $\choicefun$. For any option set $\optset$, this property considers the set
\vspace{-7pt}
\begin{equation*}
\posi(\optset)\coloneqq
\cset[\bigg]{\sum_{i=1}^n\lambda_i\opt[i]}{n\in\naturals,\lambda_i>0,\opt[i]\in\optset},
\end{equation*}
of all positive linear combinations of the elements in $\optset$, and requires that if any of these positive linear combinations---any mixture---is strictly desirable, then $\optset$ itself should contain a strictly desirable option as well.
\begin{definition}%[Mixingness]
\label{def:mixingrejects}
A set of desirable option sets $\rejectset$ is \emph{mixing} if it satisfies
\begin{enumerate}[label=$\mathrm{K}_{\mathrm{M}}$.,ref=$\mathrm{K}_{\mathrm{M}}$,leftmargin=*]
\item\label{ax:rejects:removepositivecombinations} if $\altoptset\in\rejectset$ and $\optset\subseteq\altoptset\subseteq\posi\group{\optset}$, then also $\optset\in\rejectset$, for all $\optset,\altoptset\in\optsets$;
\end{enumerate}
A choice function $\choicefun$ is called \emph{mixing} if $\rejectset[\choicefun]$ is.
%The set of all mixing sets of desirable options is denoted by~$\convcohrejectsets$.
\end{definition}

As the following result shows, mixingness achieves exactly what we need: for any coherent set of desirable option sets $\rejectset$, it guarantees that the coherent lower previsions in $\cohlowprevs(\rejectset)$ are in fact linear.
\begin{proposition}\label{prop:mixing:dominatinglowerprevarelinear}
Let $\rejectset$ be a coherent set of desirable option sets that is mixing. Then for any $\lowprev\in\cohlowprevs(\rejectset)$, we have that $\lowprev\in\linprevs$. Hence, $\cohlowprevs(\rejectset)=\linprevs(\rejectset)$.
\end{proposition}

By combining this result with Theorem~\ref{theo:axiomatisationofArchimedeanity}, it follows that Archimedean choice functions that are mixing correspond to E-admissibility. The next result formalizes this and furthermore shows that the converse is true as well.%, and that the largest set of linear previsions for which this correspondence with E-admissibility holds is given by $\linprevs(\rejectset[\choicefun])$.

\begin{theorem}%[Axiomatisation of E-admissibility]
\label{theo:axiomatisationforEadmissibility}
A choice function $\choicefun$ is Archimedean and mixing if and only if there is a non-empty set $\mathcal{P}\subseteq\linprevs$ of linear previsions such that $\choicefun=\choicefun[\mathcal{P}]^{\,\mathrm{E}}$. 
The largest such set~$\mathcal{P}$ is then\/ $\linprevs(\rejectset[\choicefun])$.
\end{theorem}

\section{Axiomatising Maximality}\label{sec:Maximality}

Having axiomatised E-admissibility, we now proceed to do the same for maximality. The link with Archimedeanity is not that obvious here though, because there is no immediate connection between Equations~\eqref{eq:archimedeanchoicefunction} and~\eqref{eq:choicefromMaximality}. Rather than focus on how to relate these two equations, we therefore zoom in on the properties of maximality itself. One such property, which is often used to illustrate the difference with E-admissibility, is that a choice function that corresponds to maximality is completely determined by its restriction to so-called \emph{binary} choices---that is, choices between two options.

\begin{definition}%[Binarity]
\label{def:mixingrejects}
A choice function $\choicefun$ is \emph{binary} if for all $\optset\in\optsets$ and $\opt\in\optset$:
\begin{equation*}
\opt\in\choicefun(\optset)
\Leftrightarrow
(\forall\altopt\in\optset\setminus\{\opt\})~\opt\in\choicefun(\{\opt,\altopt\})
\end{equation*}
%A choice function $\choicefun$ is called \emph{mixing} if $\rejectset[\choicefun]$ is.
%The set of all mixing sets of desirable options is denoted by~$\convcohrejectsets$.
\end{definition}

Inspired by this observation, we impose binarity as an additional axiom, alongside Archimedeanity. As the following result shows, these two conditions are necessary and sufficient for a choice function $\choicefun$ to be of the form $\choicefun[\mathcal{P}]^{\,\mathrm{M}}$, hence providing an axiomatisation for decision making with maximality.

\begin{theorem}%[Axiomatisation of maximality]
\label{theo:axiomatisationforMaximality}
A choice function $\choicefun$ is Archimedean and binary if and only if there is a non-empty set $\mathcal{P}\subseteq\linprevs$ of linear previsions such that $\choicefun=\choicefun[\mathcal{P}]^{\,\mathrm{M}}$. 
The largest such set~$\mathcal{P}$ is then\/ $\linprevs(\rejectset[\choicefun])$.
\end{theorem}

The formal proof of this result is rather technical, but the basic idea behind the sufficiency proof is nevertheless quite intuitive. First, for every $\opt\in\choicefun(\optset)$, the binarity of $\choicefun$ implies that $\opt\in\choicefun(\{\opt,\altopt\})$ for every $\altopt\in\optset\setminus\{\opt\}$. For every such $\altopt\in\optset\setminus\{\opt\}$, since $\choicefun$ is Archimedean, Theorem~\ref{theo:axiomatisationofArchimedeanity} furthermore implies that there is a coherent lower prevision $\lowprev$ such that $\lowprev(\altopt-\opt)\leq0$. Because of Theorem~\ref{theo:lowerenvelop}, this in turn implies that there is a linear prevision such that $\linprev(\altopt-\opt)\leq0$ and therefore also $\linprev(\opt)\geq\linprev(\altopt)$. The challenging part consists in showing that $\linprev\in\linprevs(\rejectset[\choicefun])$ and establishing necessity.

\section{An Axiomatisation for Maximising Expected Utility}\label{sec:axiomatisingEM}

As we have seen in Sections~\ref{sec:Eadmissibility} and~\ref{sec:Maximality}, mixingness and binarity have quite a different effect on Archimedean choice functions. The former implies that they correspond to E-admissibility, while the latter leads to maximality. What is intriguing though is that the set of linear previsions $\mathcal{P}$ is twice the same. Indeed, as can be seen from Theorem~\ref{theo:axiomatisationforEadmissibility} and~\ref{theo:axiomatisationforMaximality}, we may assume without loss of generality that this set is equal to $\linprevs(\rejectset[\choicefun])$. For a choice function $\choicefun$ that is mixing \emph{and} binary, we therefore find that $\choicefun=\choicefun[\mathcal{P}]^{\,E}=\choicefun[\mathcal{P}]^{\,M}$, with $\mathcal{P}=\linprevs(\rejectset[\choicefun])$. As the following result shows, this is only possible if $\mathcal{P}$ is a singleton.

\begin{proposition}\label{prop:EisMiffsingle}
Let $\mathcal{P}\subseteq\linprevs$ be a non-empty set of linear previsions. Then $\choicefun[\mathcal{P}]^{\,E}=\choicefun[\mathcal{P}]^{\,M}$ if and only if $\mathcal{P}=\{P\}$ consists of a single linear prevision $P\in\linprevs$.
\end{proposition}

As a fairly immediate consequence, we obtain the following axiomatic characterisation of choice functions that correspond to maximising expected utility.

\begin{theorem}%[Axiomatisation of maximising expected utility]
\label{theo:axiomatisationforEM}
A choice function $\choicefun$ is Archimedean, binary and mixing if and only if there is a linear prevision $\linprev\in\linprevs$ such that $\choicefun=\choicefun[\linprev]$.
\end{theorem}

\section{Conclusion and Future Work}

The main conclusion of this work is that choice functions, when interpreted in terms of (strict) desirability, can provide an axiomatic basis for decision making with sets of probability models. In particular, we were able to derive necessary and sufficient conditions for a choice function to correspond to either E-admissibility or maximality. As a byproduct, we also obtained a characterisation for choice functions that correspond to maximising expected utility.

The key concept on which these results were based is that of an Archimedean choice function, where Archimedeanity is itself a combination of several conditions. The first of these conditions is translation invariance; this condition is fairly simple and allows for a reduction from choice functions to sets of desirable option sets. The resulting set of desirable option sets should then satisfy two more conditions: coherence and~\ref{ax:infiniteoptionsets:evencontinuity}. Coherence is also fairly simple, because it follows directly from the principles of desirability. The condition \ref{ax:infiniteoptionsets:evencontinuity}, however, is more involved, making it perhaps the least intuitive component of Archimedeanity.

The abstract character of \ref{ax:infiniteoptionsets:evencontinuity} is not intrinsic to the property itself though, but rather to the framework on which it is imposed. In fact, the basic principle \ref{ax:strictdesirability:epsilon} on which~\ref{ax:infiniteoptionsets:evencontinuity} is based is very simple: if $\opt$ is strictly desirable, then there must be some positive real $\epsilon$ such that $\opt-\epsilon$ is strictly desirable as well. The reason why this simplicity does not translate to \ref{ax:infiniteoptionsets:evencontinuity} is because we restrict attention to option sets that are finite. Consider for example an assessment of the form $\{\opt\}\in\rejectset$. This means that $\opt$ is strictly desirable and therefore implies, due to \ref{ax:strictdesirability:epsilon}, that the option set $\{\opt-\epsilon\colon\epsilon\in\posreals\}$ contains at least one strictly desirable option. Hence, we should simply impose that this set belongs to $\rejectset$. This is not possible though because $\{\opt-\epsilon\colon\epsilon\in\posreals\}$ is infinite, while our framework of sets of desirable option sets only considers finite option sets.

This situation can be remedied, and the axiom of Archimedeanity can be simplified, by developing and adopting a framework of sets of desirable options that allows for infinite option sets, and connecting it to a theory of choice funtions that chooses from possibly infinite option sets. Explaining how this works is beyond the scope and size of the present contribution though; I intend to report on those results elsewhere.

\subsubsection*{Acknowlegements.}

This work was funded by the BOF starting grant 01N04819 and is part of a larger research line on choice functions of Gert de Cooman and I~\cite{debock2018,pmlr-v103-de-bock19b}. Within this line of research, this contribution is one of two parallel papers on Archimedean choice functions, one by each of us. My paper---this one---deals with the case where options are bounded real-valued functions. It axiomatises Archimedean choice functions from the ground up by starting from (strict) desirability principles and proves that the resulting axioms guarantee a representation in terms of coherent lower prevision. The paper of Gert~\cite{ipmu2020decooman:arxiv} defines Archimedeanity directly in terms of coherent lower previsions---superlinear functionals, actually, in his case---but considers the more general case where options live in an abstract Banach space; he also extends the concept of a coherent lower prevision to this more general context and discusses the connection with horse lotteries. We would like to combine our respective results in future work.

\bibliographystyle{splncs04}
\bibliography{archimedeanCF.bib}




% LNCS Homepage, \url{http://www.springer.com/lncs}. Last accessed 4
% \end{thebibliography}

% \section{Stuff I Removed For Now}

% fixed filter implies one lower prevision

% otherwise (or rather, in all cases) with respect to a set of previsions; in fact, the set of all dominating ones

% Even convexity ook uitleggen en refereren naar Fabio. Fabio toont trouwens ook aan dat voor maximality, $\mathcal{P}$ zonder verlies aan algemeenheid evenly convex kan gekozen worden.

% Perhaps introduce $\linprevs(\choicefun)$ directly, and similarly for $\cohlowprevs(\choicefun)$, and prove that it does not matter whether you use $\rejectset[\choicefun]$ as an intermediate step?
\iftoggle{arxiv}{
\newpage

\appendix

\section{Proofs and Intermediate Results}

The defining properties~\ref{ax:lowprev:inf}--\ref{ax:lowprev:superadditive} of a coherent lower prevision imply various additional properties as well~\cite[Section 2.6.1]{walley1991}; we will need the following two in our proofs:

\begin{enumerate}[label=$\mathrm{LP}_{\arabic*}$.,ref=$\mathrm{LP}_{\arabic*}$,leftmargin=*,start=4]
\item\label{ax:lowprev:constant} $\lowprev(\mu)=\mu$ for all constant options $\mu\in\reals$;
\item\label{ax:lowprev:constantadditivity} $\lowprev(\opt+\mu)=\lowprev(\opt)+\mu$ for all $\mu\in\reals$ and $\opt\in\gbls$.
\end{enumerate}

With any coherent lower prevision $\lowprev$, we can also associate a corresponding coherent upper prevision $\uppprev$, defined by
\begin{equation*}
\uppprev(\opt)\coloneqq-\lowprev(-\opt)
\text{~for all $\opt\in\gbls$.}
\end{equation*}
Linear previsions are lower previsions for which $\lowprev$ and $\uppprev$ coincide.
\begin{proposition}\cite[Section 2.3.6]{walley1991}\label{prop:lineariff}
A coherent lower prevision $\lowprev$ on $\gbls$ is a linear prevision if and only $\lowprev(\opt)=\uppprev(\opt)$ for all $\opt\in\gbls$. 
\end{proposition}

\begin{lemma}\label{lemma:rejectlowprevcoherent}
For any $\lowprev\in\cohlowprevs$, $\rejectset[\lowprev]$ is coherent.
\end{lemma}
\begin{proof}
\ref{ax:rejects:removezero} and~\ref{ax:rejects:nozero} hold because $\lowprev(0)=0$ due to~\ref{ax:lowprev:constant}. \ref{ax:rejects:pos} holds because for any $\opt\in\gbls$ with $\inf\opt>0$, $\lowprev(\opt)\geq\inf\opt>0$ due to~\ref{ax:lowprev:inf}. Since~\ref{ax:rejects:mono} holds trivially, it remains to prove~\ref{ax:rejects:cone}. So consider any $\optset[1],\optset[2]\in\rejectset[\lowprev]$ and, for all $\opt\in\optset[1]$ and $\altopt\in\optset[2]$, some $(\lambda_{\opt,\altopt},\mu_{\opt,\altopt})>0$. Since $\optset[1]\in\rejectset[\lowprev]$, there is some $\opt^*\in\optset[1]$ such that $\lowprev(\opt^*)>0$. Similarly, there is some $\altopt^*\in\optset[2]$ such that $\lowprev(\altopt^*)>0$. It therefore follows from~\ref{ax:lowprev:superadditive} and~\ref{ax:lowprev:homo} that
\begin{equation*}
\lowprev(\lambda_{\opt^*,\altopt^*}\opt^*+\mu_{\opt^*,\altopt^*}\altopt^*)
\geq\lambda_{\opt^*,\altopt^*}\lowprev(\opt^*)+\mu_{\opt^*,\altopt^*}\lowprev(\altopt^*)>0,
\end{equation*}
implying that
\begin{equation*}
\cset{\lambda_{\opt,\altopt}\opt+\mu_{\opt,\altopt}\altopt}{\opt\in\optset[1],\altopt\in\optset[2]}
\in\rejectset[\lowprev].
\end{equation*}\qed
\end{proof}

\begin{proposition}\label{prop:strictlycoherentimpliesdominatedbylowprev}
Let $\rejectset$ be an Archimedean set of desirable option sets. Then for any option set $\optset\in\optsets$ such that $\optset\notin\rejectset$, there is a coherent lower prevision $\lowprev\in\cohlowprevs(\rejectset)$ such that $\optset\notin\rejectset[{\lowprev}]$.
\end{proposition}
\begin{proof}
Let $\rejectset$ be an Archimedean set of desirable option sets and consider any option set $\optset\in\optsets$ such that $\optset\notin\rejectset$. 

Since $\rejectset$ is Archimedean, we know from Definition~\ref{def:infinite:strictcoherence} that there is some $\epsilon\in\posreals^{\rejectset}$ such that $\optset\notin\natex(\rejectset[\epsilon])$. Because of Equation~\eqref{eq:natex}, this in turn implies that there must be some coherent set of desirable options $\desirset\in\cohdesirsets$ such that $\optset\notin\rejectset[\desirset]$ and $\rejectset[\epsilon]\subseteq\rejectset[\desirset]$.
For this coherent set of desirable options, we now consider the corresponding coherent lower prevision $\lowprev$, defined by
\begin{equation}\label{eq:fromDtoLP}
\lowprev(\opt)\coloneqq\sup\{\mu\in\reals\colon\opt-\mu\in\desirset\}
\text{~for all $\opt\in\gbls$.}
\end{equation}
That this functional $\lowprev$ is indeed a coherent lower prevision is a standard result and follows easily from the coherence of $\desirset$; see for example~\cite[Section 2.3.3]{walley1991}. It remains to show that $\optset\notin\rejectset[\lowprev]$ and $\rejectset\subseteq\rejectset[\lowprev]$.

First assume \emph{ex absurdo} that $\optset\in\rejectset[\lowprev]$. It then follows from Equation~\eqref{eq:Kfromlowprev} that there is an option $\opt\in\optset$ such that $\lowprev(\opt)>0$. Because of Equation~\eqref{eq:fromDtoLP}, this implies that there is some real $\mu>0$ such that $\opt-\mu\in\desirset$. Since $\mu>0$ and \ref{ax:desirs:pos} imply that $\mu\in\desirset$, it therefore follows from \ref{ax:desirs:cone} that $\opt=(\opt-\mu)+\mu\in\desirset$. Since $\opt\in\optset$, this allows us to infer from Equation~\eqref{eq:desirset:to:rejectset} that $\optset\in\rejectset[\desirset]$, a contradiction. Hence, it follows that $\optset\notin\rejectset[\lowprev]$.

Finally, to show that $\rejectset\subseteq\rejectset[\lowprev]$, we consider any $\altoptset\in\rejectset$ and prove that $\altoptset\in\rejectset[\lowprev]$.
Consider any $\opt\in\optset$.
 Since $\altoptset\in\rejectset$, it follows from Equation~\eqref{eq:Kepsilon} that $\altoptset-\epsilon(\altoptset)\in\rejectset[\epsilon]$. Since $\rejectset[\epsilon]\subseteq\rejectset[\desirset]$, this implies that $\altoptset-\epsilon(\altoptset)\in\rejectset[\desirset]$. It therefore follows from Equation~\eqref{eq:desirset:to:rejectset} that there is some $\altopt\in\altoptset-\epsilon(\altoptset)$ such that $\altopt\in\desirset$, or equivalently, some $\opt\in\altoptset$ such that $\opt-\epsilon(\altoptset)\in\desirset$. Since $\epsilon(\altoptset)>0$, it therefore follows from Equation~\eqref{eq:fromDtoLP} that $\lowprev(\opt)>0$. Since $\opt\in\altoptset$, we infer from Equation~\eqref{eq:Kfromlowprev} that $\altoptset\in\rejectset[\lowprev]$.
\qed
\end{proof}

\begin{proof*}{Proof of Theorem~\ref{theo:rejectsets:representation:lowerprev:twosidedstrict}}
For the `only if' part of the statement, we assume that $\rejectset\in\rejectsets$ is an Archimedean set of desirable option sets.

For any $\optset\in\optsets$ such that $\optset\notin\rejectset$, we know from Proposition~\ref{prop:strictlycoherentimpliesdominatedbylowprev} that there is some $\lowprev[\optset]\in\cohlowprevs(\rejectset)$ such that $\optset\notin\rejectset[{\lowprev[\optset]}]$. Since $\{0\}\notin\rejectset$ due to~\ref{ax:rejects:nozero}, this implies that $\cohlowprevs(\rejectset)$ is non-empty because it contains $\lowprev[\{0\}]$. Furthermore, since $\optset\notin\rejectset[{\lowprev[\optset]}]$ for every $\optset\notin\rejectset$, with $\lowprev[\optset]\in\cohlowprevs(\rejectset)$, we also know that $\optset\notin\bigcap\cset{\rejectset[\lowprev]}{\lowprev\in\cohlowprevs(\rejectset)}=\rejectset[{\cohlowprevs(\rejectset)}]$ for every $\optset\notin\rejectset$, and therefore, that $\rejectset[{\cohlowprevs(\rejectset)}]\subseteq\rejectset$. Since it follows from the definition of $\cohlowprevs(\rejectset)$ that also $\rejectset\subseteq\rejectset[{\cohlowprevs(\rejectset)}]$, we find that $\rejectset=\rejectset[{\cohlowprevs(\rejectset)}]$. Since $\cohlowprevs(\rejectset)\subseteq\cohlowprevs$ is non-empty, it follows that the `only if' part of the statement holds for $\mathcal{P}=\cohlowprevs(\rejectset)$. 

For the `if' part of the statement, we consider any non-empty set $\mathcal{P}\subseteq\cohlowprevs$ of coherent lower previsions such that $\rejectset=\rejectset[\mathcal{P}]$. Since $\rejectset[\mathcal{P}]=\bigcap\cset{\rejectset[\lowprev]}{\lowprev\in\mathcal{P}}$, it then follows that $\rejectset\subseteq\rejectset[\lowprev]$ for all $\lowprev\in\mathcal{P}$, implying that $\mathcal{P}$ is a subset of $\cohlowprevs(\rejectset)$. Hence, the largest such set $\mathcal{P}$ is indeed $\cohlowprevs(\rejectset)$. It remains to show that for any such set $\mathcal{P}$, $\rejectset=\rejectset[\mathcal{P}]$ is Archimedean, which is what we now set out to do.

We know from Lemma~\ref{lemma:rejectlowprevcoherent} that $\rejectset[\lowprev]$ is coherent for every $\lowprev\in\mathcal{P}$. Since $\rejectset=\rejectset[\mathcal{P}]=\bigcap\cset{\rejectset[\lowprev]}{\lowprev\in\mathcal{P}}$ and $\mathcal{P}$ is non-empty, and because coherence is clearly preserved under taking non-empty intersections, it follows that $\rejectset$ coherent. According to Definition~\ref{def:infinite:strictcoherence}, the Archimedeanity of $\rejectset$ therefore hinges on whether or not it satisfies~\ref{ax:infiniteoptionsets:evencontinuity}. We will show that it indeed satisfies this property. 
To that end, consider any $\optset\in\optsets$ such that $\optset\in\natex(\rejectset[\epsilon])$ for all $\epsilon\in\posreals^{\rejectset}$. We need to prove that $\optset\in\rejectset$.

Consider any $\lowprev\in\mathcal{P}$ and any $\altoptset\in\rejectset=\rejectset[\mathcal{P}]\subseteq\rejectset[\lowprev]$. It then follows from Equation~\eqref{eq:Kfromlowprev} that there is some $\opt\in\altoptset$ such that $\lowprev(\opt)>0$. Let $\epsilon^*(\altoptset)\coloneqq\nicefrac{\lowprev(\opt)}{2}>0$. It then follows from~\ref{ax:lowprev:constantadditivity} that $\lowprev(\opt-\epsilon^*(\altoptset))=\lowprev(\opt)-\epsilon^*(\altoptset)=\nicefrac{\lowprev(\opt)}{2}>0$, which, since $\opt\in\altoptset$, implies that $\altoptset-\epsilon^*(\altoptset)\in\rejectset[\lowprev]$. Since this is true for every $\altoptset\in\rejectset$, we have found some $\smash{\epsilon^*\in\posreals^{\rejectset}}$ such that $\rejectset[\epsilon^*]\subseteq\rejectset[\lowprev]$. Hence, since $\rejectset[\lowprev]$ is coherent because of Lemma~\ref{lemma:rejectlowprevcoherent}, it follows from Equation~\eqref{eq:natex} that $\natex(\rejectset[\epsilon^*])\subseteq\rejectset[\lowprev]$. Since by assumption, $\optset\in\natex(\rejectset[\epsilon])$ for all $\smash{\epsilon\in\posreals^{\rejectset}}$, we therefore have in particular that $\smash{\optset\in\natex(\rejectset[\epsilon^*])\subseteq\rejectset[\lowprev]}$. Since this is true for every $\lowprev\in\mathcal{P}$, we conclude that $\optset\in\bigcap\cset{\rejectset[\lowprev]}{\lowprev\in\mathcal{P}}=\rejectset[\mathcal{P}]=\rejectset$. \qed
\end{proof*}

\begin{proof*}{Proof of Proposition~\ref{prop:strictlycoherentCiffK}}
First assume that $\choicefun$ is an Archimedean choice function, meaning that $\rejectset[\choicefun]$ is Archimedean and $\choicefun$ is translation invariant. Consider the Archimedean set of desirable option sets $\rejectset=\rejectset[\choicefun]$. For any $\optset\in\optsets$ and $\opt\in\optset$, we then find that
\begin{align*}
\opt\in\choicefun[\rejectset](\optset)
\Leftrightarrow
\optset\ominus\opt\notin\rejectset
\Leftrightarrow
\optset\ominus\opt\notin\rejectset[\choicefun]
&\Leftrightarrow
0\in\choicefun\big((\optset\ominus\opt)\cup\{0\}\big)\\
&\Leftrightarrow
0\in\choicefun(\optset-\opt)
\Leftrightarrow
\opt\in\choicefun(\optset),
\end{align*}
using Equation~\eqref{eq:fromCtoK} for the first equivalence, Equation~\eqref{eq:fromKtoC} for the third one, and translation invariance---Equation~\eqref{eq:translationinvariance}---for the last. Hence, $\choicefun=\choicefun[\rejectset]$.

Next, consider any Archimedean set of desirable option sets $\rejectset$ such that $\choicefun=\choicefun[\rejectset]$. First observe that $\choicefun[\rejectset]$ is translation invariant because, for all $\optset\in\optsets$ and $\opt\in\optset$,
\begin{equation*}
\opt\in\choicefun[\rejectset](\optset)
\Leftrightarrow
\optset\ominus\opt\notin\rejectset
\Leftrightarrow
(\optset-\opt)\ominus0\notin\rejectset
\Leftrightarrow
0\in\choicefun[\rejectset](\optset-\opt),
\end{equation*}
using Equation~\eqref{eq:fromKtoC} for the first and third equivalence. Next, observe that for all $\optset\in\optsets$
\begin{align*}
\optset\in\rejectset[\choicefun]
&\Leftrightarrow
0\notin\choicefun(\optset\cup\{0\})\\
&\Leftrightarrow
0\notin\choicefun[\rejectset](\optset\cup\{0\})
\Leftrightarrow
(\optset\cup\{0\})\ominus0\in\rejectset
\Leftrightarrow
\optset\setminus\{0\}\in\rejectset
\Leftrightarrow
\optset\in\rejectset,
\end{align*}
using Equation~\eqref{eq:fromCtoK} for the first equivalence, Equation~\eqref{eq:fromKtoC} for the third equivalence and~\ref{ax:rejects:removezero} and~\ref{ax:rejects:mono} for the fifth one. This implies that $\rejectset=\rejectset[\choicefun]$, thereby implying that $\rejectset$ is unique and equal to $\rejectset[\choicefun]$. Since $\rejectset$ is Archimedean, this also implies that $\rejectset[\choicefun]$ is Archimedean and therefore, since we already know that $\choicefun$ is translation invariant, that $\choicefun$ is Archimedean.
\qed
\end{proof*}

\begin{proof*}{Proof of Theorem~\ref{theo:axiomatisationofArchimedeanity}}
First assume that $\choicefun$ is Archimedean. It then follows from Proposition~\ref{prop:strictlycoherentCiffK} that $\choicefun=\choicefun[\rejectset]$ for the Archimedean set of desirable option sets $\rejectset=\rejectset[\choicefun]$. Since $\rejectset[\choicefun]$ is Archimedean, it furthermore follows from Theorem~\ref{theo:rejectsets:representation:lowerprev:twosidedstrict} that $\rejectset[\choicefun]=\rejectset[\mathcal{P}]$, with $\mathcal{P}=\cohlowprevs(\rejectset[\choicefun])$. Hence, we find that $\choicefun=\choicefun[\rejectset]=\choicefun[{\rejectset[\choicefun]}]=\choicefun[{\rejectset[\mathcal{P}]}]$, which is equal to $\choicefun[\mathcal{P}]$ because
\begin{align*}
\choicefun[{\rejectset[\mathcal{P}]}](\optset)
&=
\{\opt\in\optset\colon\optset\ominus\opt\notin\rejectset[\mathcal{P}]\}\\
&=
\big\{\opt\in\optset\colon\hspace{-2pt}(\exists\lowprev\in\mathcal{P})\,\optset\ominus\opt\notin\rejectset[\lowprev]\big\}\\
&=
\big\{\opt\in\optset\colon\hspace{-2pt}(\exists\lowprev\in\mathcal{P})(\forall\altopt\in\optset\setminus\{\opt\})\,\lowprev(\altopt-\opt)\leq0\big\}=\choicefun[\mathcal{P}](\optset)\text{~~for all $\optset\in\optsets$.}
\end{align*}
Conversely, consider any non-empty set $\mathcal{P}\subseteq\cohlowprevs$ of coherent lower previsions such that $\choicefun=\choicefun[\mathcal{P}]$. It then follows from the derivation above that $\choicefun=\choicefun[{\rejectset[\mathcal{P}]}]$ and from Theorem~\ref{theo:rejectsets:representation:lowerprev:twosidedstrict} that $\rejectset[\mathcal{P}]$ is Archimedean. Proposition~\ref{prop:strictlycoherentCiffK} therefore implies that $\choicefun$ is Archimedean.

It remains to show that $\mathcal{P}\subseteq\cohlowprevs(\rejectset[\choicefun])$. To that end, consider any $\lowprev\in\mathcal{P}$ and any $\optset\in\rejectset[\choicefun]$. Since $\optset\in\rejectset[\choicefun]$, we know that $0\notin\choicefun(\optset\cup\{0\})$. Since $\choicefun=\choicefun[\mathcal{P}]$, it therefore follows from Equation~\eqref{eq:archimedeanchoicefunction} that there is some $\altopt\in(\optset\cup\{0\})\setminus\{0\}$ such that $0<\lowprev(\altopt-0)=\lowprev(\altopt)$. Since $(\optset\cup\{0\})\setminus\{0\}=\optset\setminus\{0\}\subseteq\optset$, this implies that there is some $\altopt\in\optset$ such that $\lowprev(\altopt)>0$, meaning that $\optset\in\rejectset[\lowprev]$. Since $\optset\in\rejectset[\choicefun]$ was arbitrary, this implies that $\rejectset[\choicefun]\subseteq\rejectset[\lowprev]$, hence $\lowprev\in\cohlowprevs(\rejectset[\choicefun])$. Since $\lowprev\in\mathcal{P}$ was arbitrary, we find that $\mathcal{P}$ is indeed a subset of $\cohlowprevs(\rejectset[\choicefun])$.
\qed
\end{proof*}

\begin{proof*}{Proof of Proposition~\ref{prop:mixing:dominatinglowerprevarelinear}}
Consider any $\lowprev\in\cohlowprevs(\rejectset)$ and assume \emph{ex absurdo} that $\lowprev\notin\linprevs$. It then follows from Proposition~\ref{prop:lineariff} that there is some $\opt\in\gbls$ such that $\lowprev(\opt)\neq\uppprev(\opt)$, or equivalently, such that $\lowprev(\opt)+\lowprev(-\opt)\neq0$. Since we know from \ref{ax:lowprev:constant} and \ref{ax:lowprev:superadditive} that
\begin{equation*}
0=\lowprev(0)=\lowprev(\opt-\opt)\geq\lowprev(\opt)+\lowprev(-\opt),
\end{equation*}
this implies that $\lowprev(\opt)+\lowprev(-\opt)<0$. Let $\epsilon\coloneqq-\lowprev(\opt)-\lowprev(-\opt)>0$ and define $\opt[1]\coloneqq\opt-\lowprev(\opt)-\nicefrac{\epsilon}{2}$ and $\opt[2]\coloneqq-\opt+\lowprev(\opt)+\epsilon$. Then on the one hand, it follows from~\ref{ax:lowprev:constantadditivity} that
\begin{equation*}
\lowprev(\opt[1])
=\lowprev(\opt-\lowprev(\opt)-\nicefrac{\epsilon}{2})
=\lowprev(\opt)-\lowprev(\opt)-\nicefrac{\epsilon}{2}
=-\nicefrac{\epsilon}{2}<0
\end{equation*}
and
\vspace{-2pt}
\begin{equation*}
\lowprev(\opt[2])
=\lowprev(-\opt+\lowprev(\opt)+\epsilon)
=\lowprev(-\opt)+\lowprev(\opt)+\epsilon
=0,\vspace{5pt}
\end{equation*}
so $\{\opt[1],\opt[2]\}\notin\rejectset[\lowprev]$ and therefore, since $\rejectset\subseteq\rejectset[\lowprev]$, also $\{\opt[1],\opt[2]\}\notin\rejectset$. On the other hand, however, since $\opt[1]+\opt[2]=\nicefrac{\epsilon}{2}>0$, it follows from the coherence of $\rejectset$---and \ref{ax:rejects:pos} and~\ref{ax:rejects:mono} in particular---that $\{\opt[1],\opt[2],\opt[1]+\opt[2]\}\in\rejectset$. Since $\{\opt[1],\opt[2]\}\subseteq\{\opt[1],\opt[2],\opt[1]+\opt[2]\}\subseteq\posi(\{\opt[1],\opt[2]\})$, this contradicts the mixingness of $\rejectset$.
\qed
\end{proof*}

\begin{proof*}{Proof of Theorem~\ref{theo:axiomatisationforEadmissibility}}
First assume that $\choicefun$ is Archimedean and mixing. This implies that the set of desirable option sets $\rejectset[\choicefun]$ is Archimedean (and hence also coherent) and mixing, and therefore, it follows from Proposition~\ref{prop:mixing:dominatinglowerprevarelinear} that $\cohlowprevs(\rejectset[\choicefun])=\linprevs(\rejectset[\choicefun])$. Since $\choicefun$ is Archimedean, it furthermore follows from Theorem~\ref{theo:axiomatisationofArchimedeanity} that $\choicefun=\choicefun[\mathcal{P}]$, with $\mathcal{P}=\cohlowprevs(\rejectset[\choicefun])=\linprevs(\rejectset[\choicefun])$ non-empty. For every $\optset\in\optsets$, this implies that
\begin{align}
\choicefun[\mathcal{P}](\optset)
&=
\big\{\opt\in\optset\colon(\exists\linprev\in\mathcal{P})\,(\forall\altopt\in\optset\setminus\{\opt\})~\linprev(\altopt-\opt)\leq0\big\}\notag\\
&=
\big\{\opt\in\optset\colon(\exists\linprev\in\mathcal{P})\,(\forall\altopt\in\optset\setminus\{\opt\})~\linprev(\opt)\geq\linprev(\altopt)\big\}
=\choicefun[\mathcal{P}]^{\mathrm{\,E}}(\optset),\label{eq:droptheE}
\end{align}
where the crucial second equality holds because the elements of $\mathcal{P}=\linprevs(\rejectset[\choicefun])$ are \emph{linear} previsions. Hence, $\choicefun=\choicefun[\mathcal{P}]^{\mathrm{\,E}}$.

Consider now any non-empty set $\mathcal{P}\subseteq\linprevs$ of linear previsions such that $\choicefun=\choicefun[\mathcal{P}]^{\mathrm{\,E}}$. It then follows from Equation~\eqref{eq:droptheE} that $\choicefun=\choicefun[\mathcal{P}]^{\mathrm{\,E}}=\choicefun[\mathcal{P}]$. Since $\mathcal{P}\subseteq\linprevs\subseteq\cohlowprevs$, it therefore follows from Theorem~\ref{theo:axiomatisationofArchimedeanity} that $\choicefun$ is Archimedean. To show that it is mixing, we consider any $\optset,\altoptset\in\optsets$ such that $\altoptset\in\rejectset[\choicefun]$ and $\optset\subseteq\altoptset\subseteq\posi(\optset)$. We need to prove that then also $\optset\in\rejectset[\choicefun]$. Since $\altoptset\in\rejectset[\choicefun]$, we know that $0\notin\choicefun(\altoptset\cup\{0\})=\choicefun[\mathcal{P}]^{\mathrm{\,E}}(\altoptset\cup\{0\})$. Consider now any $\linprev\in\mathcal{P}$. It then follows from Equation~\eqref{eq:choicefromEadmissibility} that there is some $\altopt\in(\altoptset\cup\{0\})\setminus\{0\}=\altoptset\setminus\{0\}$ such that $\linprev(\altopt)>\linprev(0)=0$, using \ref{ax:lowprev:constant} for the last equality. Since $\altopt\in\altoptset\setminus\{0\}\subseteq\altoptset\subseteq\posi(\optset)$, we know that $\altopt=\sum_{i=1}^n\lambda_i\opt[i]$, with $n\in\naturals$ and, for all $i\in\{1,\dots,n\}$, $\lambda_i>0$ and $\opt[i]\in\optset$. Since $\linprev$ is linear and $\linprev(\altopt)>0$, this implies that $\sum_{i=1}^n\lambda_i\linprev(\opt[i])>0$. Since the coefficients $\lambda_i$ are all strictly positive, this implies that there must be some $k\in\{1,\dots,n\}$ such that $\linprev(\opt[k])>0=\linprev(0)$. 
Since $\opt[k]\in\optset$, this implies that we have found some $\opt\in\optset$ such that $\linprev(\opt)>0$. Since $\linprev(0)=0$, we also know that $\opt\neq0$. 
Hence, $\opt\in\optset\setminus\{0\}=(\optset\cup\{0\})\setminus\{0\}$ and $\linprev(\opt)>0=\linprev(0)$. 
Since we can repeat this argument for every $\linprev\in\mathcal{P}$, we infer from Equation~\eqref{eq:choicefromEadmissibility} that $0\notin\choicefun[\mathcal{P}]^{\mathrm{\,E}}(\optset\cup\{0\})=\choicefun(\optset\cup\{0\})$. Hence, $\optset\in\rejectset[\choicefun]$, as desired.

It remains to show that $\mathcal{P}\subseteq\linprevs(\rejectset[\choicefun])$. On the one hand, since $\choicefun=\choicefun[\mathcal{P}]$ and $\mathcal{P}\subseteq\linprevs\subseteq\cohlowprevs$, we know from Theorem~\ref{theo:axiomatisationofArchimedeanity} that $\mathcal{P}\subseteq\cohlowprevs(\rejectset[\choicefun])$. On the other hand, since $\choicefun$ is Archimedean and mixing, we know that $\rejectset[\choicefun]$ is Archimedean (and hence coherent) and mixing, so Proposition~\ref{prop:mixing:dominatinglowerprevarelinear} tells us that $\cohlowprevs(\rejectset[\choicefun])=\linprevs(\rejectset[\choicefun])$. Hence, we find that $\mathcal{P}\subseteq\cohlowprevs(\rejectset[\choicefun])=\linprevs(\rejectset[\choicefun])$.
\qed
\end{proof*}

\begin{proposition}\label{prop:seperatesingletonswithP}
Let $\rejectset$ be an Archimedean set of desirable option sets and consider any $\opt\in\gbls$. Then $\{\opt\}\notin\rejectset$ if and only if there is some $\linprev\in\linprevs(\rejectset)$ such that $\{\opt\}\notin\rejectset[\linprev]$.
\end{proposition}
\begin{proof}
First assume that $\{\opt\}\notin\rejectset$. Since $\rejectset$ is Archimedean, it then follows from Proposition~\ref{prop:strictlycoherentimpliesdominatedbylowprev} that there is a coherent lower prevision $\lowprev\in\cohlowprevs(\rejectset)$ such that $\{\opt\}\notin\rejectset[\lowprev]$, meaning that $\lowprev(\opt)\leq0$. Because of Theorem~\ref{theo:lowerenvelop}, this implies that there is some $\linprev\in\linprevs$ such that $\linprev(\opt)=\lowprev(\opt)\leq0$ and, for all $\altopt\in\gbls$, $\linprev(\altopt)\geq\lowprev(\altopt)$. Since $\linprev(\opt)\leq0$, we see that $\{\opt\}\notin\rejectset[\linprev]$. To establish the `only if' part of the statement, it therefore remains to prove that $\linprev\in\linprevs(\rejectset)$, or equivalently, that $\rejectset\subseteq\rejectset[\linprev]$. To that end, consider any $\optset\in\rejectset$. 
Since $\lowprev\in\cohlowprevs(\rejectset)$, we know that $\rejectset\subseteq\rejectset[\lowprev]$ and therefore, that $\optset\in\rejectset[\lowprev]$, meaning that there is some $\altopt\in\optset$ such that $\lowprev(\altopt)>0$. Since $\linprev(\altopt)\geq\lowprev(\altopt)$, this implies that $\linprev(\altopt)>0$. 
Hence, since $\altopt\in\optset$, we find that $\optset\in\rejectset[\linprev]$. Since $\optset\in\rejectset$ was arbitrary, it follows that $\rejectset\subseteq\rejectset[\linprev]$.

For the `if' part of the statement, assume that there is some $\linprev\in\linprevs(\rejectset)$ such that $\{\opt\}\notin\rejectset[\linprev]$. Since $\rejectset$ is Archimedean, it follows from Theorem~\ref{theo:rejectsets:representation:lowerprev:twosidedstrict} that $\rejectset=\rejectset[\mathcal{P}]=\bigcap\cset{\rejectset[\lowprev]}{\lowprev\in\mathcal{P}}$, with $\mathcal{P}=\cohlowprevs(\rejectset)$. Since $\linprev\in\linprevs(\rejectset)\subseteq\cohlowprevs(\rejectset)=\mathcal{P}$ and $\{\opt\}\notin\rejectset[\linprev]$, this implies that $\{\opt\}\notin\rejectset$.
\qed
\end{proof}

\begin{proof*}{Proof of Theorem~\ref{theo:axiomatisationforMaximality}}
First assume that $\choicefun$ is Archimedean and binary.
Since we know that $\choicefun$ is Archimedean, it follows from Proposition~\ref{prop:strictlycoherentCiffK} that $\choicefun=\choicefun[\rejectset]$ for the Archimedean (and hence coherent) set of desirable option sets $\rejectset=\rejectset[\choicefun]$. For any $\opt,\altopt\in\gbls$ such that $\opt\neq\altopt$, this implies that
\begin{align*}
\opt\in\choicefun(\{\opt,\altopt\})
\Leftrightarrow
\opt\in\choicefun[\rejectset](\{\opt,\altopt\})
&\Leftrightarrow
\{\altopt-\opt\}=\{\opt,\altopt\}\ominus\{\opt\}\notin\rejectset=\rejectset[\choicefun]\\
&\Leftrightarrow
(\exists\linprev\in\linprevs(\rejectset[\choicefun]))
\,
\{\altopt-\opt\}\notin\rejectset[\linprev]\\
&\Leftrightarrow
(\exists\linprev\in\linprevs(\rejectset[\choicefun]))
\,
\linprev(\altopt-\opt)\leq0\\
&\Leftrightarrow
(\exists\linprev\in\linprevs(\rejectset[\choicefun]))
\,
\linprev(\opt)\geq\linprev(\altopt),
\end{align*}
using Proposition~\ref{prop:seperatesingletonswithP} for the third equivalence. Since $\choicefun$ is binary, this implies that
\begin{align*}
\choicefun(\optset)
&=\big\{\opt\in\optset\colon(\forall\altopt\in\optset\setminus\{\opt\})~\opt\in\choicefun(\{\opt,\altopt\})\big\}\\
&=
\big\{\opt\in\optset\colon(\forall\altopt\in\optset\setminus\{\opt\})(\exists\linprev\in\linprevs(\rejectset[\choicefun]))\,\linprev(\opt)\geq\linprev(\altopt)\big\}
=\choicefun[{\linprevs(\rejectset[\choicefun])}]^{\mathrm{\,M}}(\optset)
\end{align*}
for all $\optset\in\optsets$. Hence, $\choicefun=\choicefun[\mathcal{P}]^{\mathrm{\,M}}$, with $\mathcal{P}=\linprevs(\rejectset[\choicefun])$. To see why $\linprevs(\rejectset[\choicefun])$ is non-empty, it suffices to note that $\{0\}\notin\rejectset[\choicefun]$ because of \ref{ax:rejects:nozero} and the fact that $\rejectset[\choicefun]$ is coherent. Indeed, it then follows from Proposition~\ref{prop:seperatesingletonswithP} that there is some $\linprev\in\linprevs(\rejectset[\choicefun])$ such that $\{0\}\notin\rejectset[\linprev]$, which implies that $\linprevs(\rejectset[\choicefun])$ is non-empty.

Consider now any non-empty set $\mathcal{P}\subseteq\linprevs$ of linear previsions such that $\choicefun=\choicefun[\mathcal{P}]^{\mathrm{\,M}}$. For all $\optset\in\optsets$ and $\opt\in\optset$, it then follows from Equation~\eqref{eq:choicefromMaximality} that
\begin{align*}
\opt\in\choicefun(\optset)
\Leftrightarrow
\opt\in\choicefun[\mathcal{P}]^{\mathrm{\,M}}(\optset)
&\Leftrightarrow(\forall\altopt\in\optset\setminus\{\opt\})(\exists\linprev\in\mathcal{P})\,\linprev(\opt)\geq\linprev(\altopt)\\
&\Leftrightarrow(\forall\altopt\in\optset\setminus\{\opt\})~\opt\in\choicefun[\mathcal{P}]^{\mathrm{\,M}}(\{\opt,\altopt\})\\
&\Leftrightarrow(\forall\altopt\in\optset\setminus\{\opt\})~\opt\in\choicefun(\{\opt,\altopt\}).
\end{align*}
Hence, $\choicefun$ is binary. To show that it is also Archimedean, we consider the set 
\begin{equation*}
\mathcal{P}^*\coloneqq\big\{\min_{i=1}^n\linprev[i]\colon n\in\naturals,(\forall i\in\{1,\dots,n\})\,\linprev[i]\in\mathcal{P}\big\}
\end{equation*}
of all minima of a finite number of linear previsions in $\mathcal{P}$, which is clearly non-empty because $\mathcal{P}$ is. Due to Theorem~\ref{theo:lowerenvelop}, we also know that $\mathcal{P}^*\subseteq\cohlowprevs$. We will show that $\choicefun[\mathcal{P}^*]=\choicefun[\mathcal{P}]^{\mathrm{\,M}}$. This implies that $\choicefun=\choicefun[\mathcal{P}]^{\mathrm{\,M}}=\choicefun[\mathcal{P}^*]$ and therefore, because of Theorem~\ref{theo:axiomatisationofArchimedeanity}, that $\choicefun$ is Archimedean.

Consider any $\optset\in\optsets$ and any $\opt\in\optset$. First assume that $\opt\in\choicefun[\mathcal{P}^*](\optset)$. This means that there is some $\lowprev\in\mathcal{P}^*$ such that, for all $\altopt\in\optset\setminus\{\opt\}$, $\lowprev(\altopt-\opt)\leq0$. For all $\altopt\in\optset\setminus\{\opt\}$, since $\lowprev\in\mathcal{P}^*$, this implies that there is some $\linprev\in\mathcal{P}$ such that $\linprev(\altopt-\opt)=\lowprev(\altopt-\opt)\leq0$ and therefore also $\linprev(\opt)\geq\linprev(\altopt)$. Hence, $\opt\in\choicefun[\mathcal{P}]^{\mathrm{\,M}}(\optset)$. Conversely, assume that $\opt\in\choicefun[\mathcal{P}]^{\mathrm{\,M}}(\optset)$. For all $\altopt\in\optset\setminus\{\opt\}$, this implies that there is some $\linprev[\altopt]\in\mathcal{P}$ such that $\linprev[\altopt](\opt)\geq\linprev[\altopt](\altopt)$ and therefore $\linprev[\altopt](\altopt-\opt)\leq0$. Let $\lowprev=\min_{\altopt\in\optset\setminus\{\opt\}}\linprev[\altopt]$. Then $\lowprev\in\mathcal{P}^*$ because $\optset\setminus\{\opt\}$ is finite and, for all $\altopt\in\optset\setminus\{\opt\}$, $\lowprev(\altopt-\opt)\leq\linprev[\altopt](\altopt-\opt)\leq0$. Hence, $\opt\in\choicefun[\mathcal{P}^*](\optset)$. We conclude that $\opt\in\choicefun[\mathcal{P}^*](\optset)$ if and only if $\opt\in\choicefun[\mathcal{P}]^{\mathrm{\,M}}(\optset)$. Since $\optset\in\optsets$ and $\opt\in\optset$ are arbitrary, this implies that, indeed, $\choicefun[\mathcal{P}^*]=\choicefun[\mathcal{P}]^{\mathrm{\,M}}$.

It remains to show that $\mathcal{P}\subseteq\linprevs(\rejectset[\choicefun])$. To that end, consider any $\linprev\in\mathcal{P}$ and any $\optset\in\rejectset[\choicefun]$. Since $\optset\in\rejectset[\choicefun]$, we know that $0\notin\choicefun(\optset\cup\{0\})$. Since $\choicefun=\choicefun[\mathcal{P}]^{\mathrm{\,M}}$, it therefore follows from Equation~\eqref{eq:choicefromMaximality} that there is some $\altopt\in(\optset\cup\{0\})\setminus\{0\}$ such that $0=\linprev(0)<\linprev(\altopt)$. Since $(\optset\cup\{0\})\setminus\{0\}=\optset\setminus\{0\}\subseteq\optset$, this implies that there is some $\altopt\in\optset$ such that $\linprev(\altopt)>0$, meaning that $\optset\in\rejectset[\linprev]$. Since $\optset\in\rejectset[\choicefun]$ was arbitrary, this implies that $\rejectset[\choicefun]\subseteq\rejectset[\linprev]$, hence $\linprev\in\linprevs(\rejectset[\choicefun])$. Since $\linprev\in\mathcal{P}$ was arbitrary, we find that $\mathcal{P}$ is indeed a subset of $\linprevs(\rejectset[\choicefun])$.
\qed
\end{proof*}

\begin{proof*}{Proof of Proposition~\ref{prop:EisMiffsingle}}
First assume that $\choicefun[\mathcal{P}]^{\,E}=\choicefun[\mathcal{P}]^{\,M}$. Consider any two $\linprev[1],\linprev[2]\in\mathcal{P}$ and assume \emph{ex absurdo} that $\linprev[1]\neq\linprev[2]$. This implies that there is some $\opt\in\gbls$ such that $\linprev[1](\opt)\neq\linprev[2](\opt)$. Without loss of generality, we can assume that $\linprev[1](\opt)>\linprev[2](\opt)$ (otherwise, it suffices to reverse the role of $\linprev[1]$ and $\linprev[2]$). Let $\opt[1]\coloneqq\linprev[1](\opt)-\opt$ and $\opt[2]\coloneqq\opt-\linprev[2](\opt)$ and consider the set $\optset\coloneqq\{0,\opt[1],\opt[2]\}$. It then follows from \ref{ax:lowprev:constantadditivity} that $\linprev[1](\opt[2])=\linprev[1](\opt)-\linprev[2](\opt)>0$, $\linprev[2](\opt[1])=\linprev[1](\opt)-\linprev[2](\opt)>0$ and from \ref{ax:lowprev:constant} that $\linprev[1](0)=\linprev[2](0)=0$. Hence, $\opt[1]\neq0$ and $\opt[2]\neq0$, so $\optset\setminus\{0\}=\{\opt[1],\opt[2]\}$. \ref{ax:lowprev:constantadditivity} and~\ref{ax:lowprev:constant} also imply that $\linprev[1](\opt[1])=0=\linprev[1](0)$ and $\linprev[2](\opt[2])=0=\linprev[2](0)$. Since $\linprev[1],\linprev[2]\in\mathcal{P}$, it therefore follows from Equation~\eqref{eq:choicefromMaximality} that $0\in\choicefun[\mathcal{P}]^{\mathrm{\,M}}(\optset)$. Since $\choicefun[\mathcal{P}]^{\,E}=\choicefun[\mathcal{P}]^{\,M}$, this implies that $0\in\choicefun[\mathcal{P}]^{\mathrm{\,E}}(\optset)$, and therefore, because of Equation~\eqref{eq:choicefromEadmissibility}, that there is some $\linprev\in\mathcal{P}$ such that, for all $\altopt\in\optset\setminus\{0\}=\{\opt[1],\opt[2]\}$, $\linprev(0)\geq\linprev(\altopt)$. Due to the linearity of $\linprev$ and \ref{ax:lowprev:constant}, this implies that
\begin{equation*}
\linprev(\opt[1]+\opt[2])=\linprev(\opt[1])+\linprev(\opt[2])\leq\linprev(0)+\linprev(0)=0+0=0.
\end{equation*}
However, since $\opt[1]+\opt[2]=\linprev[1](\opt)-\opt+\opt-\linprev[2](\opt)=\linprev[1](\opt)-\linprev[2](\opt)$, \ref{ax:lowprev:constant} also implies that
\begin{equation*}
\linprev(\opt[1]+\opt[2])
=\linprev\big(\linprev[1](\opt)-\linprev[2](\opt)\big)=\linprev[1](\opt)-\linprev[2](\opt)>0,
\end{equation*}
a contradiction. Hence, it must be that $\linprev[1]=\linprev[2]$. Since this is true for any two $\linprev[1],\linprev[2]\in\mathcal{P}$, and since $\mathcal{P}$ is non-empty, this implies that $\mathcal{P}$ consists of a single linear prevision $\linprev\in\linprevs$: $\mathcal{P}=\{\linprev\}$.

Next, assume that there is a linear prevision $\linprev\in\linprevs$ such that $\mathcal{P}=\{\linprev\}$. It then follows directly from Equation~\eqref{eq:choicefromEadmissibility} and~\eqref{eq:choicefromMaximality} that $\choicefun[\mathcal{P}]^{\mathrm{\,E}}=\choicefun[\mathcal{P}]^{\mathrm{\,M}}$. 
\end{proof*}

\begin{proof*}{Proof of Theorem~\ref{theo:axiomatisationforEM}}
First assume that $\choicefun$ is Archimedean, binary and mixing. Let $\mathcal{P}=\linprevs(\rejectset[\choicefun])$. Since $\choicefun$ is Archimedean and binary, it follows from Theorem~\ref{theo:axiomatisationforMaximality} that $\choicefun=\choicefun[\mathcal{P}]^{\mathrm{\,M}}$ and that $\mathcal{P}$ is non-empty. Since $\choicefun$ is Archimedean and mixing, it follows from Theorem~\ref{theo:axiomatisationforEadmissibility} that $\choicefun=\choicefun[\mathcal{P}]^{\mathrm{\,E}}$. Hence, $\choicefun[\mathcal{P}]^{\mathrm{\,E}}=\choicefun[\mathcal{P}]^{\mathrm{\,M}}$. Since $\mathcal{P}$ is non-empty, it therefore follows from Proposition~\ref{prop:EisMiffsingle} that there is a unique linear prevision $\linprev\in\linprevs$ such that $\mathcal{P}=\{P\}$. Equations~\eqref{eq:choicefromExpectationmax} and~\eqref{eq:choicefromEadmissibility} now imply that $\choicefun=\choicefun[\mathcal{P}]^{\mathrm{\,E}}=\choicefun[\linprev]$.

Next, consider a linear prevision $\linprev\in\linprevs$ and let $\choicefun=\choicefun[\linprev]$. Consider the set $\mathcal{P}=\{\linprev\}$. On the one hand, it then follows from Equations~\eqref{eq:choicefromExpectationmax} and~\eqref{eq:choicefromEadmissibility} that $\choicefun=\choicefun[\linprev]=\choicefun[\mathcal{P}]^{\mathrm{\,E}}$ and therefore, because of Theorem~\ref{theo:axiomatisationforEadmissibility}, that $\choicefun$ is Archimedean and mixing. On the other hand, it follows from Equations~\eqref{eq:choicefromExpectationmax} and~\eqref{eq:choicefromMaximality} that $\choicefun=\choicefun[\linprev]=\choicefun[\mathcal{P}]^{\mathrm{\,M}}$ and therefore, because of Theorem~\ref{theo:axiomatisationforMaximality}, that $\choicefun$ is binary.
\end{proof*}
}{
}

\end{document}